\documentclass[preprint,3p,times]{elsarticle}
\usepackage{hyperref}
\usepackage{amsmath}
\usepackage{mathrsfs}
\usepackage{tikz}
\usepackage{slashbox}
\usepackage{hyperref}
\usepackage{listings}
\usepackage{multirow}
\usepackage{amssymb}
\usepackage{amsthm}
\usepackage{epstopdf}
\usepackage{algorithm}
\usepackage{algpseudocode}
\usepackage{subfigure}
\usepackage{tabularx,booktabs}
\usepackage{caption}
\usepackage{verbatim}

\biboptions{sort&compress}
\theoremstyle{plain}

\newtheorem{lemma}{Lemma}[section]
\newtheorem{proposition}{Proposition}[section]

\theoremstyle{definition}
\newtheorem{definition}{Definition}[section]

\newtheorem{example}{Example}[section]
\newtheorem{remark}{Remark}[section]

\makeatletter
\newif\if@borderstar
\def\bordermatrix{\@ifnextchar*{%
		\@borderstartrue\@bordermatrix@i}{\@borderstarfalse\@bordermatrix@i*}%
}
\def\@bordermatrix@i*{\@ifnextchar[{\@bordermatrix@ii}{\@bordermatrix@ii[()]}}
\def\@bordermatrix@ii[#1]#2{%
	\begingroup
	\m@th\@tempdima8.75\p@\setbox\z@\vbox{%
		\def\cr{\crcr\noalign{\kern 2\p@\global\let\cr\endline }}%
		\ialign {$##$\hfil\kern 2\p@\kern\@tempdima &\thinspace %
			\hfil $##$\hfil &&\quad\hfil $##$\hfil\crcr\omit\strut %
			\hfil\crcr\noalign{\kern -\baselineskip}#2\crcr\omit %
			\strut\cr}}%
	\setbox\tw@\vbox{\unvcopy\z@\global\setbox\@ne\lastbox}%
	\setbox\tw@\hbox{\unhbox\@ne\unskip\global\setbox\@ne\lastbox}%
	\setbox\tw@\hbox{%
		$\kern\wd\@ne\kern -\@tempdima\left\@firstoftwo#1%
		\if@borderstar\kern2pt\else\kern -\wd\@ne\fi%
		\global\setbox\@ne\vbox{\box\@ne\if@borderstar\else\kern 2\p@\fi}%
		\vcenter{\if@borderstar\else\kern -\ht\@ne\fi%
			\unvbox\z@\kern-\if@borderstar2\fi\baselineskip}%
		\if@borderstar\kern-2\@tempdima\kern2\p@\else\,\fi\right\@secondoftwo#1$%
	}\null \;\vbox{\kern\ht\@ne\box\tw@}%
	\endgroup
}
\makeatother
\allowdisplaybreaks

\begin{document}
\begin{frontmatter}	
	\title{$(O,G)$-granular variable precision fuzzy rough sets based on overlap and grouping functions}
	\author{Wei Li\fnref{label1}}
	\ead{weili\_1998@163.com}
	\author{Bin Yang\fnref{label1}\corref{cor1}}
	\ead{binyang0906@whu.edu.cn,binyang0906@nwsuaf.edu.cn}
	\author{Junsheng Qiao\fnref{label2}}
	\ead{jsqiao@nwnu.edu.cn}
	\address[label1]{College of Science, Northwest A \& F University, Yangling 712100, PR China}
	\address[label2]{College of Mathematics and Statistics, Northwest Normal University, Lanzhou 730070, PR China}
	\cortext[cor1]{Corresponding author.}
	
	\begin{abstract}
		Since Bustince et al. introduced the concepts of overlap and grouping functions, these two types of aggregation functions have attracted a lot of interest in both theory and applications.
		In this paper, the depiction of $(O,G)$-granular variable precision fuzzy rough sets ($(O,G)$-GVPFRSs for short) is first given based on overlap and grouping functions.
		Meanwhile, to work out the approximation operators efficiently, we give another expression of upper and lower approximation operators by means of fuzzy implications and co-implications.
	    Furthermore, starting from the perspective of construction methods, $(O,G)$-GVPFRSs are represented under diverse fuzzy relations.
	     Finally, some conclusions on the granular variable precision fuzzy rough sets (GVPFRSs for short) are extended to $(O,G)$-GVPFRSs under some additional conditions.
	\end{abstract}
	
\end{frontmatter}

\section{Introduction}\label{section0}
\subsection{Brief review of fuzzy rough sets}
Rough set, as a way to portray uncertainty problems, was originally proposed by Polish mathematician Pawlak in 1982 \cite{Pawlak1982Roughset,Pawlak1991RoughSet}, and it has been extensively developed in the fields of knowledge discovery \cite{Skowron1998RoughSet} and  data mining.
Rough set theory uses indistinguishable relations to divide the knowledge of research domain, thus forming a system of knowledge representation that approximates an arbitrary subset of the universe by defining upper and lower approximation operators \cite{Chen2006Roughapproximations}.
As a generalization of the classical theory, Zadeh introduced the fuzzy set theory  \cite{Zadeh1965Fuzzysets} in 1965, where objects can be owned by different sets with different membership functions.
 Since rough sets are defined based on equivalence relations, they are mainly used to process qualitative (discrete) data \cite{Jensen2004Fuzzy-roughattributes}, and there are greater restrictions on the processing of real-valued data sets in the database. In particular, fuzzy sets can solve this problem by dealing with fuzzy concepts. Therefore, complementing the features of rough sets and fuzzy sets with each other constitutes a new research hotspot.

 In 1990, Dubois and Prade \cite{Dubois1990Roughfuzzy} described fuzzy rough sets, which is the combination of two uncertainty models, and then extended the fundamental properties to fuzzy rough sets.
 As another innovation of rough set,
 Ziarko presented the variable precision rough set \cite{Ziarko1993Variableprecision}, which mainly solved the classification problem of uncertain and inaccurate information with an effective error-tolerance competence.
 More details about variable precision rough sets can refer to  \cite{Mieszkowicz-Rolka2004Remarkson,Mieszkowicz-Rolka2004Variable precision,Zhang2008Variableprecision}.
 In addition, since the upper and lower approximation operators of fuzzy rough sets are defined according to membership functions, while rough sets are described based on the union of some sets, there exists significant difference in the granular structure of the two.
 To overcome this limitation, Chen et al. \cite{Chen2011Granularcomputing} explored the concept and related properties of granular fuzzy sets based on fuzzy similarity relations. Furthermore, from the perspective of granular computing, the granular fuzzy set is used to characterize the granular structure of upper and lower approximations.
 However, the above model cannot tolerate even small errors and is not suited to handle uncertain information well.
 Some extended fuzzy rough sets are applied to solve these problem, but some studies still have problems in dealing with mislabeled samples (see, e.g., \cite{Hu2010Softfuzzy,Hu2012Onrobust,Zhao2009Themodel}), and others have only considered the relative error cases. \cite{FernandezSalido2003Roughset,Mieszkowicz-Rolka2004Remarkson}).

To fill these loopholes, the model of variable precision $(\theta, \sigma)$-fuzzy rough sets over fuzzy granules were presented by Yao et al. \cite{Yao2014Anovelvariable}. However,
the above model is based on fuzzy $*$-similarity relation, satisfying reflexivity, symmetry and $*$-transitivity, which is too strict to facilitate generalized conclusions.
Thus, Wang and Hu \cite{Wang2015Granularvariable} studied the GVPFRSs and then the equivalent expressions of the approximation operators are given with fuzzy implications and co-implications over arbitrary fuzzy relations.
Subsequently, they gave the properties of GVPFRSs on different fuzzy relations.
 In addition, compared with unit interval, the complete lattice has a wider structure, so Qiao and Hu expanded the content of \cite{Wang2015Granularvariable} and \cite{Yao2014Anovelvariable}, and further discussed the concept of granular variable precision $L$-fuzzy rough sets based on residuated lattices.

 In fact, both \cite{Qiao2016Granularvariable} and \cite{Wang2015Granularvariable} are based on $t$-norm ($t$-conorm), which satisfying associative, commutative, increasing in each argument and has a identity element 1 (resp. 0). However, there are various applications \cite{Fodor1995Nonstandardconjunctions,Bustince2010Overlap,Bustince2012Grouping} in which the associativity property of the $t$-norm (resp. $t$-conorm) is not necessary, such as classification problems, face recognition and image processing.

\subsection{Brief analysis of overlap and grouping functions}
Bustince et al. described the axiomatic definitions of overlap and grouping functions \cite{Bustince2009Overlapindex,Bustince2012Grouping}, which stem from some practical problems in image processing and classification.
In fact, in some situations, the associativity of $t$-norm and $t$-conorm usually does not work. Therefore, as two types of noncombining fuzzy logic connectives, overlap and grouping functions have made rapid development in theoretical research and practical applications.

In theory, there exists many studies involving overlap and grouping functions, such as crucial properties \cite{Bedregal2013Newresults,Dimuro2014Archimedeanoverlap,Wang2019Themodularity}, corresponding implications \cite{Dimuro2014OnGNimplications,Dimuro2015Onresidual,Ti2018OnONcoimplications}, additive generator pairs \cite{Dimuro2016Onadditive}, interval overlap functions and grouping functions \cite{Bedregal2017Generalized interval-valued,Qiao2017Oninterval}, distributive equations \cite{Liu2020New results,Zhang2020ondistributive,Zhang2021onthedistributivity} and concept extensions \cite{DeMiguel2019Generaloverlap,Zhou2021migrativityproperties}.
From an application point of view, overlap and grouping functions can find interesting applications in  classifications \cite{Lucca2017CCintegrals,Paternain2016Capacitiesand}, image processing \cite{Bustince2007Imagethresholding,Bustince2010Overlap,Jurio2013Someproperties}, fuzzy community detection problems \cite{Bustince2007Imagethresholding} and decision making \cite{Bustince2012Grouping,Elkano2018Consensusvia}.

\subsection{The motivation of this paper}
\begin{figure}[h]
	\centering
	\includegraphics[width=12cm]{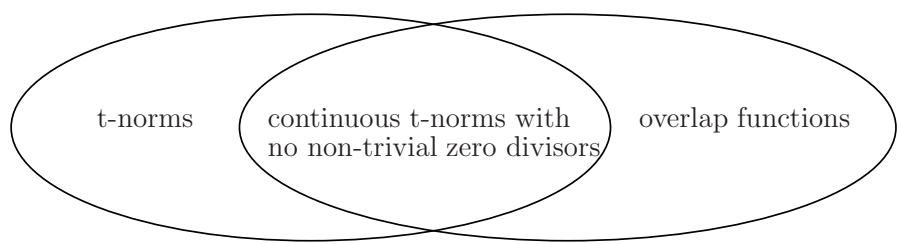}\\
	\caption{The relationship between $t$-norms and overlap functions \cite{Qiao2019Ondistributive}}\label{f1}
\end{figure}

In~\cite{Adamatti2014InterdisciplinaryApplications}, the authors have pointed out
that $O: [0, 1]^{2}\longrightarrow[0, 1]$ is an associative overlap function (resp. grouping
function) if and only if $O$ is a continuous and positive $t$-norm (resp. $t$-conorm).
 On the other side, we note that overlap and grouping functions can be considered as another extension of  classical logical connective $\land$  and $\vee$ on the unit interval, which differ from $t$-norms and $t$-conorms.
 Hence, we can use them to replace the classical logical operators and then define the granular variable precision approximation operators. Meanwhile, from the application aspect, the study of fuzzy rough sets based on overlap and grouping functions has a pivotal role in practical problems.
Therefore, based on aforementioned consideration, and as a supplement of the GVPFRSs \cite{Wang2015Granularvariable},
this paper continues the studies in $(O,G)$-GVPFRSs based on overlap and grouping functions instead of $t$-norm and $t$-conorm. It should be pointed out that the present paper further enriches the application of overlap and grouping functions. In addition, it makes the research on fuzzy rough sets more complete.

The rest of this paper is arranged as follows. Section 2 enumerates some fundamental concepts that are necessary to understand this paper.
Section 3 proposes the $(O,G)$-GVPFRSs with general fuzzy relations and gives an alternative expression for efficient computation of the approximation operators. Furthermore, we study the $(O,G)$-GVPFRSs under the conditions of crisp relations and crisp sets and draw the corresponding conclusions.
Section 4 represents the $(O,G)$-GVPFRSs on diverse fuzzy relations. In particular, some special conclusions are given under some additional conditions. Section 5, conclusions on our research are given.

\section{Preliminaries}\label{section2}
In this section, we recapitulate some fundamental notions which shall be used in the sequel.

\begin{definition}\label{d:overlap}(\cite{Bustince2010Overlap})
	An overlap function is a binary function $O: [0, 1]^{2}\longrightarrow[0, 1]$ which satisfies the following conditions for all $x,y \in [0,1]$:
	\begin{flushleft}
		(O1) $O(x, y) = O(y, x)$;\\
		(O2) $O(x, y)=0$ iff $xy=0$;\\
		(O3) $O(x, y)=1$ iff $xy=1$;\\
		(O4) $O$ is non-decreasing;\\
		(O5) $O$ is continuous.
	\end{flushleft}

	\begin{flushleft}
		\qquad Furthermore, an overlap function $O$ fulfills the exchange principle (\cite{Dimuro2015Onresidual}) if
		
		(O6) $\forall x,y,u\in [0,1]: O(x,O(y,u))=O(y,O(x,u)).$
	\end{flushleft}
\end{definition}

\begin{definition}\label{d:grouping}(\cite{Bustince2012Grouping})
	A grouping function is a binary function $G: [0, 1]^{2}\longrightarrow[0, 1]$ which satisfies the following conditions for all $x,y \in [0,1]$:
	\begin{flushleft}
		(G1) $G(x, y) = G(y, x)$;\\
		(G2) $G(x, y)=0$ iff $x=y=0$;\\
		(G3) $G(x, y)=1$ iff $x=1$ or $y=1$;\\
		(G4) $G$ is non-decreasing;\\
		(G5) $G$ is continuous.
	\end{flushleft}
	
	\begin{flushleft}
		\qquad Furthermore, a grouping function $G$ fulfills the exchange principle (\cite{Dimuro2015Onresidual}) if
		
		(G6) $\forall x,y,u\in [0,1]: G(x,G(y,u))=G(y,G(x,u)).$
	\end{flushleft}
\end{definition}

\begin{remark}\label{exchange property &associative}(\cite{Dimuro2015Onresidual})
	Notice that a commutative function $H:[0,1]^2\longrightarrow [0,1]$ is associative if and only if $H$ satisfies the exchange principle.
	It is obvious that an overlap function $O$ (resp. a grouping function $G$) is associative if and only if it satisfies (O6) (resp. (G6)).
\end{remark}

\begin{remark}\label{exchange property &identity}(\cite{Dimuro2015Onresidual,Dimuro2016Onadditive})
	Suppose overlap function $O$ satisfies (O6), then 1 is the identity element of $O$, similarly, when a grouping function $G$ satisfies (G6), then 0 is the identity element of $G$.
\end{remark}

 Next, some common overlap and grouping functions are listed in  \cite{Bedregal2013Newresults,Dimuro2016Onadditive}.
\begin{example}
	
	\begin{enumerate}[(1)]
		\item Any continuous $t$-norm with no non-trivial zero divisors is an overlap function.
		\item The function $O_{p}: [0, 1]^{2}\longrightarrow[0, 1]$ given by $$O_{p}(x, y)=x^{p}y^{p}$$ is an overlap function for any $p>0$ and $p\neq 1$. Since it neither satisfies the associative law nor takes $1$ as identity element, it is not a $t$-norm.
		\item The function $O_{DB}: [0, 1]^{2}\longrightarrow[0, 1]$ given by
		$$O_{DB}=\left\{
		\begin{array}{ll}
		\frac{2xy}{x+y}, & \hbox{if $x+y\neq 0$,} \\
		0, & \hbox{if $x+y=0$}
		\end{array}
		\right.$$
		is an overlap function.
		\item Any continuous $t$-conorm with no non-trivial one divisors is a grouping function.
		\item The function $G_{p}: [0, 1]^{2}\longrightarrow[0, 1]$ given by $$G_{p}(x, y)=1-(1-x)^{p}(1-y)^{p}$$ is a grouping function for $p>1$.
		Since it neither satisfies the associative law nor takes $0$ as identity element, it is not a $t$-conorm.
	\end{enumerate}
\end{example}

In the following, we give the definitions of fuzzy implication and fuzzy co-implication on the basis of overlap and grouping function.

A fuzzy implication $I_{O}: [0, 1]^{2}\longrightarrow[0, 1]$ given by
$$I_{O}(x, y)=\max\{z\in[0, 1]: O(x, z)\leq y\}$$
for all $x, y\in[0, 1]$. In~\cite{Dimuro2015Onresidual},
Dimuro et al. have proved $O$ and $I_{O}$
form an adjoint pair, if they satisfy the residuation property:
$$\forall x, y, u\in [0, 1]: O(x, u)\leq y\Leftrightarrow I_{O}(x, y)\geq u.$$

Furthermore, $I_O$ satisfies the exchange principle \cite{Dimuro2015Onresidual} if and only if
$$\forall x,y,z\in [0,1],\ I_O(x,I_O(y,z))=I_O(y,I_O(x,z)).$$

Fuzzy implication $I_{O}$ was introduced in \cite{Dimuro2015Onresidual} and fuzzy co-implication $I^{G}$ was discussed in \cite{Baets1997Coimplicators}. Furthermore, since $O$ and $G$ are dual w.r.t. $N$, we can deduce the properties of fuzzy co-implication $I^{G}$ easily.

A fuzzy co-implication $I^{G}: [0, 1]^{2}\longrightarrow[0, 1]$ given by
$$I^{G}(x, y)=\min\{z\in[0, 1]: y\leq G(x, z)\}$$
for all $x, y\in[0, 1]$. Similarly, the following hold:
$$\forall x, y, u\in[0, 1]: y\leq G(x, u)\Leftrightarrow I^{G}(x, y)\leq u.$$

Furthermore, $I^G$ satisfies the exchange principle if and only if
$$\forall x,y,z\in [0,1],\ I^G(x,I^G(y,z))=I^G(y,I^G(x,z)).$$

If $O$ and $G$ are dual w.r.t. $N$,
then for all $x, y\in[0, 1]$,
\begin{align*}
I_{O}(x, y)&=N(I^{G}(N(x), N(y))), \\
I^{G}(x, y)&=N(I_{O}(N(x), N(y))).
\end{align*}
According to the definition of $I_O$ that for all $x, y\in[0, 1]$,
\begin{align*}
I_{O}(x, y)&=\max\{z\in[0, 1]: O(x, z)\leq y\}\\
&=\max\{z\in[0, 1]: N(G(N(x), N(z)))\leq y\}\\
&=\max\{z\in[0, 1]: G(N(x), N(z))\geq N(y)\}\\
&=\max\{N(z)\in [0, 1]: G(N(x), z)\geq N(y)\}\\
&=N(\min\{z\in[0, 1]: G(N(x), z)\geq N(y)\})\\
&=N(I^{G}(N(x), N(y))).
\end{align*}
Similarly, the following equation can be obtained.
\begin{align*}
I^{G}(x, y)&=\min\{z\in[0, 1]: y\leq G(x, z)\}\\
&=\min\{z\in[0, 1]: y\leq N(O(N(x), N(z)))\}\\
&=\min\{z\in[0, 1]: N(y)\geq O(N(x), N(z))\}\\
&=\min\{N(z)\in [0, 1]: N(y)\geq O(N(x), z)\}\\
&=N(\max\{z\in[0, 1]: N(y)\geq O(N(x), z)\})\\
&=N(I_{O}(N(x), N(y))).
\end{align*}

\begin{remark}(\cite{Dimuro2015Onresidual})\label{r:2.3}
	$I_O$ satisfies the exchange property if and only if $O$ satisfies \textnormal{(O6)}, similarly, $I^G$ satisfies the exchange property if and only if $G$ satisfies \textnormal{(G6)}.
\end{remark}

\begin{lemma}\label{l:nocondition}(\cite{Qiao2021Onfuzzy})
	Let $x,y,z\in [0,1]$ and $\{x_i\}_{i\in I}\subseteq [0,1]$. Then
	\begin{enumerate}[(1)]
		\item 
		$O(x,I_O(x,y))\le y\ and\ y\le G((I^G(x,y),x))$;
		\item 
		$I_O(y,(\bigwedge_{i\in I}x_i))=\bigwedge_{i\in I}(I_O(y,x_i))\ and\ I^G(y,(\bigvee_{i\in I}x_i))=\bigvee_{i\in I}I^G(y,x_i)$;
		\item 
		$I_O(y,\bigvee_{i\in I}x_i)=\bigvee_{i\in I}I_O(y,x_i)\ and\ I^G(y,\bigwedge_{i\in I}x_i)=\bigwedge_{i\in I}I^G(y,x_i)$;
		\item 
		$I_O(\bigvee_{i\in \Lambda}x_i,y)=\bigwedge_{i\in \Lambda}I_O(x_i,y)\ and\ I^G(\bigwedge_{i\in \Lambda}x_i,y)=\bigvee_{i\in \Lambda}I^G(x_i,y)$;
		\item 
		$I_O(x,I_O(y,z))=I_O(O(x,y),z)$ iff $O$ satisfies \textnormal{(O6)} and $I^G(x,I^G(y,z))=I^G(G(x,y),z)$ iff $G$ satisfies $\textnormal{(G6)}$.
	\end{enumerate}
\end{lemma}

\begin{lemma}\label{l:neutal}(\cite{Dimuro2015Onresidual})
	Let overlap function $O$ have identity element 1, and grouping function $G$ have identity element 0. For any $x,y,z\in [0,1]$, the following statements hold.
	\begin{enumerate}[(1)]
		\item 
		$I_O(1,x)=x\ and\ I^G(0,x)=x$;
		\item 
		$x\le y\ i\!f\!f\ I_O(x,y)=1\ i\!f\!f\ I^G(y,x)=0$;
		\item 
		$x\le I_O(y,x)\ and\ x\ge I^G(y,x)$.
	\end{enumerate}
\end{lemma}

\begin{lemma}\label{l:associative law}
	Let overlap function $O:[0,1]^2\rightarrow\ [0,1]$(resp. grouping function $G:[0,1]^2\rightarrow\ [0,1]$) satisfies $\textnormal{(O6)}$ (resp. $\textnormal{(G6)}$). For any $x,y,z\in [0,1]$, the following statements hold.
	\begin{enumerate}[(1)]
		\item 
		$O(x,I_O(y,z))\le I_O(y,O(x,z))\ and\ I^G(y,G(x,z))\le G(x,I^G(y,z))$;
		\item 
		$I_O(y,z)\le I_O(I_O(x,y),I_O(x,z))\ and\ I^G(I^G(x,y),I^G(x,z))\le I^G(y,z)$.
	\end{enumerate}
\end{lemma}

\begin{proof}
	It is obvious that $O$ becomes a $t$-norm when it satisfies (O6), we can immediately obtain that $O(x,I_O(y,z))\le I_O(y,O(x,z))$ and $I_O(y,z)\le I_O(I_O(x,y),I_O(x,z))$. The equations about $G$ can be derived similarly.
\end{proof}

In the following, some basics about fuzzy sets are given.

Let finite set $X$ be universe, and the family of all fuzzy sets on $X$ is denoted $\mathscr{F}(X)$.
The fuzzy set $A$ defined as $A(x)=\alpha$ for any $A\in\mathscr{F}(X)$ and $x\in X$, is a constant and further called $\alpha_{X}$.
In addition, a fuzzy point $A$ is tagged with $y_{\alpha}$, if for all $x\in X$,
$$A(x)=\left\{
\begin{array}{ll}
\alpha, & \hbox{$x=y$;} \\
0, & \hbox{$x\neq y$;}
\end{array}
\right.$$
Furthermore, $|A|$ notes the cardinality of the set $A$ for all crisp sets $A$.

\begin{definition}\label{fuzzy negation}
	A function $N: [0, 1]\longrightarrow[0, 1]$ is a fuzzy negation, if it satisfies the following conditions:
	\begin{enumerate}[(1)]
		\item
		If $x<y$, then $N(x)>N(y)$, for all $x,y\in [0,1].$
		\item
		$N(0)=1$ and $N(1)=0$.
	\end{enumerate}
	Further, $N$ is called an
	involutive negation, if
	$N(N(x))=x$ holds for all $x\in[0, 1]$ and  the standard negation, $N(x)=1-x$ for all $x\in[0, 1]$,
	is a special case of involutive negation $N$.
\end{definition}

The operations on fuzzy sets are defined as follows: for all $A, B\in\mathscr{F}(X)$ and $x\in X$,
\begin{flushleft}
	(1)~$A^{N}(x)=N(A(x))$,\\
	(2)~$O(A, B)(x)=O(A(x), B(x))$,\\
	(3)~$G(A, B)(x)=G(A(x), B(x))$,\\
	(4)~$I_{O}(A, B)(x)=I_{O}(A(x), B(x))$,\\
	(5)~$I^{G}(A, B)(x)=I^{G}(A(x), B(x))$.
\end{flushleft}

 If for all $x, y\in[0, 1], N(x\oplus y)=N(x)\odot N(y)$, then the two binary operations $\oplus$ and $\odot$ are said to be dual with respect to (w.r.t., for short) $N$. Especially, $(A^{c})(x)=1-A(x)$ and $A\subseteq B$ defined as $A(x)\leq B(x)$ for all $x\in X$. In addition,
a fuzzy relation on $X$ is a fuzzy set $R\in \mathscr{F}(X\times X)$ and $R^{-1}$ is defined as $R^{-1}(x, y)=R(y, x)$ for all $x, y\in X$.

\begin{definition}\label{d:fuzzy relation}
	Let $R$ be a fuzzy relation on $X$ and for all $x, y, z\in X$, $R$ satisfies
	\begin{flushleft}
		(1)~seriality: $\bigvee_{y\in X}R(x, y)=1$;\\
		(2)~reflexivity: $R(x, x)=1$;\\
		(3)~symmetry: $R(x, y)=R(y, x)$;\\
		(4)~$O$-transitivity: $O(R(x, y), R(y, z))\leq R(x, z)$.
	\end{flushleft}
\end{definition}

For sake of simplicity, $\wedge$-transitive is called transitive.
$R$ ia a fuzzy $O$-preorder relation when it satisfies reflexivity and $O$-transitivity and a fuzzy $O$-similarity relation when it satisfies reflexivity, symmetry and $O$-transitivity.

Next, the model of GVPFRSs which proposed by Wang and Hu \cite{Wang2015Granularvariable} will be given below.

\begin{definition}\label{d:general model}(\cite{Wang2015Granularvariable})
	Let $R$ be a fuzzy relation on $X$, $\beta\in[0, 1]$ and
	$\mathscr{F}_{\beta}(X)=\{X_{i}\subseteq X: |X_{i}|\geq\beta|X|\}$.
	Then for all $A\in\mathscr{F}(X)$, two fuzzy
	operators $\underline{R}^{\beta}$ and 	$\overline{R}^{\beta}$ are defined as follows.
	\begin{align*}
	\underline{R}^{\beta}(A)&=\bigcup\{[x_{\gamma}]_{R}^{\bigtriangleup}: x\in X, \gamma\in[0, 1],
	\{y\in X: [x_{\gamma}]_{R}^{\bigtriangleup}(y)\leq A(y)\}\in\mathscr{F}_{\beta}(X)\},\\
	\overline{R}^{\beta}(A)&=\bigcap\{[x_{\gamma}]_{R}^{\bigtriangledown}: x\in X, \gamma\in[0, 1],
	\{y\in X: A(y)\leq [x_{\gamma}]_{R}^{\bigtriangledown}(y)\}\in\mathscr{F}_{\beta}(X)\},
	\end{align*}
	Then $\underline{R}^{\beta}$ (resp. $\overline{R}^{\beta}$)	 is the generalized granular variable precision lower (resp. upper) approximation operator and the pair $(\underline{R}^{\beta}(A), \overline{R}^{\beta}(A))$ is GVPFRSs of fuzzy set $A$.
\end{definition}

\section{$(O,G)$-granular variable precision fuzzy rough sets based on overlap and grouping functions}

In the following, we give the model of $(O,G)$-GVPFRSs and then utilize fuzzy implication and co-implication to compute the approximation operators more efficiently. In addition, we continue to study the related properties of degenerated $(O,G)$-GVPFRSs under the condition of crisp relations and crisp sets, respectively.

\begin{definition}\label{d:fuzzy granules}
	Let $R$ be a fuzzy relation on $X$. Then define the fuzzy granules $[x_{\lambda}]_{R}^{O}$ and $[x_{\lambda}]_{R}^{G}$ by
	\begin{center}
		$[x_{\lambda}]_{R}^{O}(y)=O(R(x, y), \lambda)$ and
		$[x_{\lambda}]_{R}^{G}(y)=G(R^{N}(x, y), \lambda)$,
	\end{center}
where $x, y\in X$, $\lambda\in[0, 1]$ and $N$ is an involutive negation.
\end{definition}

In \cite{Dimuro2014OnGNimplications}, Dimuro et al. have defined the class of fuzzy implications called $(G, N)$-implications, where $G$ and $N$ are grouping
functions and fuzzy negations respectively.
Detailed definition is introduced as follows:

For grouping function $G: [0, 1]^{2}\longrightarrow [0, 1]$ and fuzzy negation
$N: [0, 1]\longrightarrow [0, 1]$, the function $I_{G, N}$, denoted by
$$I_{G, N}(a, b)=G(N(a), b),$$
is a $(G, N)$-implications, where $a, b\in [0, 1]$.

Then, from the definition of $I_{G, N}$ and
Definition~\ref{d:fuzzy granules}, one concludes that
$$[x_{\lambda}]_{R}^{G}(y)=G(R^{N}(x, y), \lambda)=I_{G, N}(R(x, y), \lambda).$$

\subsection{$(O,G)$-granular variable precision fuzzy rough sets based on overlap and grouping functions}
\begin{definition}\label{d:model}
	Let $R$ be a fuzzy relation on $X$, $\beta\in[0, 1]$ and
	$\mathscr{F}_{\beta}(X)=\{X_{i}\subseteq X: |X_{i}|\geq\beta|X|\}$ such that for all $A\in\mathscr{F}(X)$,
	\begin{align*}
	\underline{R}_{O}^{\beta}(A)&=\bigcup\{[x_{\lambda}]_{R}^{O}: x\in X, \lambda\in[0, 1],
	\{y\in X: [x_{\lambda}]_{R}^{O}(y)\leq A(y)\}\in\mathscr{F}_{\beta}(X)\},\\
	\overline{R}_{G}^{\beta}(A)&=\bigcap\{[x_{\lambda}]_{R}^{G}: x\in X, \lambda\in[0, 1],
	\{y\in X: A(y)\leq [x_{\lambda}]_{R}^{G}(y)\}\in\mathscr{F}_{\beta}(X)\},
	\end{align*}
	then $\underline{R}_{O}^{\beta}$ (resp. $\overline{R}_{G}^{\beta}$) is denoted the
	$O$-granular (resp. $O$-granular) variable precision lower (resp. upper) approximation operator and the pair $(\underline{R}_{O}^{\beta}(A), \overline{R}_{G}^{\beta}(A))$ is denoted the
	$(O, G)$-granular variable precision fuzzy rough set of fuzzy set $A$.
\end{definition}
\begin{remark}
	If $t$-norm (resp. $t$-conorm) is continuous and positive,
	then Definition~\ref{d:model} in~\cite{Wang2015Granularvariable}
	is equal to $O$-granular (resp. $G$-granular
	) variable precision lower (resp. upper) approximation operator defined above.
	In this paper, $(O, G)$-GVPFRSs are defined on arbitrary fuzzy relations, where $O$ and $G$ do not need to be dual w.r.t. the standard negation $N$.
\end{remark}

In the next propositions, the equivalent statements of the
$O$-granular (resp. $G$-granular) variable precision lower (resp. upper) approximation operator will be given.

\begin{proposition}\label{p:3.1}	
	Let $R$ be a fuzzy relation on $X$. For all $A\in\mathscr{F}(X)$,
	$x\in X$ and $X_{i}\in \mathscr{F}_{\beta}(X)$, define
	\begin{align*}
	g_{A}^{(i)}(x)&=\underset{y\in X_{i}}{\bigwedge}I_{O}(R(x, y), A(y))
	\end{align*}
	\begin{align*}
	g_{A}(x)&=\underset{X_{i}\in \mathscr{F}_{\beta}(X)}{\bigvee}g_{A}^{(i)}(x).
	\end{align*}
	Then, it always holds
	$$\underline{R}_{O}^{\beta}(A)=\bigcup\{[x_{g_{A}(x)}]_{R}^{O}: x\in X\}
	\mbox{ and }
	\{y: [x_{g_{A}(x)}]_{R}^{O}(y)\leq A(y)\}\in \mathscr{F}_{\beta}(X),$$
	$where\ x\in X\ and\ A\in\mathscr{F}(X)$.
\end{proposition}
\begin{proof}
	Let $x\in X$, $\lambda\in[0, 1]$, and $\{y\in X: [x_{\lambda}]_{R}^{O}(y)\leq A(y)\}$ be written as $Y$, while
	$\{y\in X: [x_{\lambda}]_{R}^{O}(y)\leq A(y)\}\in \mathscr{F}_{\beta}(X)$.
	Then for all $y\in Y$, consider the following equivalences,
	$$[x_{\lambda}]_{R}^{O}(y)\leq A(y)\Longleftrightarrow O(R(x, y), \lambda)\leq A(y)
	\Longleftrightarrow\lambda\leq I_{O}(R(x, y), A(y)), $$
    that is $\lambda\leq g_{A}(x)$.
	Hence, for all $A\in \mathscr{F}(X)$, it always holds $\underline{R}_{O}^{\beta}(A)\subseteq\bigcup\{[x_{g_{A}(x)}]_{R}^{O}: x\in X\}$ by Definition~\ref{d:model}.
	
	Another side, for all $x\in X$, there exists $X_{i}\in \mathscr{F}_{\beta}(X)$ such that $g_{A}(x)=g_{A}^{(i)}(x)$.
	For all $y\in X_{i}$, we get that
	\begin{align*}
	[x_{g_{A}(x)}]_{R}^{O}(y)&=O(R(x, y), g_{A}^{(i)}(x))\\
	&=O(R(x, y), \underset{z\in X_{i}}{\bigwedge}I_{O}(R(x, z), A(z)))\\
	&\leq O(R(x, y), I_{O}(R(x, y), A(y)))\\
	&\leq A(y).
	\end{align*}
	Thus,
	$X_{i}\subseteq \{y\in X: [x_{g_{A}(x)}]_{R}^{O}(y)\leq A(y)\}$ and
	$\underline{R}_{O}^{\beta}(A)\supseteq\bigcup\{[x_{g_{A}(x)}]_{R}^{O}: x\in X\}$ hold.
	
	In summary,
	$\underline{R}_{O}^{\beta}(A)=\bigcup\{[x_{g_{A}(x)}]_{R}^{O}: x\in X\}$
	and
	$\{y: [x_{g_{A}(x)}]_{R}^{O}(y)\leq A(y)\}\in \mathscr{F}_{\beta}(X)$
	hold for all $x\in X$ and $A\in\mathscr{F}(X)$.
\end{proof}

\begin{proposition}\label{p:3.2}
	Let $R$ be a fuzzy relation on $X$. For all $A\in\mathscr{F}(X)$, $x\in X$ and $X_{i}\in \mathscr{F}_{\beta}(X)$, define
	\begin{align*}
	h_{A}^{(i)}(x)&=\underset{y\in X_{i}}{\bigvee}I^{G}(R^{N}(x, y), A(y))
	\end{align*}
	\begin{align*}
	h_{A}(x)&=\underset{X_{i}\in \mathscr{F}_{\beta}(X)}{\bigwedge}h_{A}^{(i)}(x).
	\end{align*}
	Then, it always holds
	$$\overline{R}_{G}^{\beta}(A)=\bigcap\{[x_{h_{A}(x)}]_{R}^{G}: x\in X\}
	\mbox{ and }
	\{y: A(y)\leq [x_{h_{A}(x)}]_{R}^{G}(y)\}\in \mathscr{F}_{\beta}(X),\
	$$
	$where\ x\in X\ and\ A\in\mathscr{F}(X)$.
\end{proposition}

\begin{proof}
	Let $x\in X$, $\lambda\in[0,1]$ and $\{y\in X: A(y)\leq[x_{\lambda}]_{R}^{G}(y)\}$ be written as $Y$,
	while $\{y\in X: A(y)\leq[x_{\lambda}]_{R}^{G}(y)\} \in \mathscr{F}_{\beta}(X)$.
	Then for all $y\in Y$, consider the following equivalences,
	\begin{align*}
	A(y)\leq[x_{\lambda}]_{R}^{G}(y)&\Longleftrightarrow A(y)\leq G(R^{N}(x, y), \lambda)
	\Longleftrightarrow I^{G}(R^{N}(x, y), A(y))\leq \lambda,
	\end{align*}
	that is $h_{A}(x)\leq \lambda$.
	Hence, for all $A\in \mathscr{F}(X)$, it always holds $\overline{R}_{G}^{\beta}(A)\supseteq \bigcap\{[x_{h_{A}(x)}]_{R}^{G}: x\in X\}$
	by Definition~\ref{d:model}.
	
	Another side, for all $x\in X$, there exists $X_{i}\in \mathscr{F}_{\beta}(X)$ such that
	$h_{A}(x)=h_{A}^{(i)}(x)$.
	For all $y\in X_{i}$, we get that
	\begin{align*}
	[x_{h_{A}(x)}]_{R}^{G}(y)&=G(R^{N}(x, y), h_{A}^{(i)}(x))\\
	&=G(R^{N}(x, y), \underset{z\in X_{i}}{\bigvee}I^{G}(R^{N}(x, z), A(z)))\\
	&\geq G(R^{N}(x, y), I^{G}(R^{N}(x, y), A(y)))\\
	&\geq A(y).
	\end{align*}
	Thus,
	$X_{i}\subseteq \{y\in X: A(y)\leq [x_{h_{A}(x)}]_{R}^{G}(y)\}$ and
	$\overline{R}_{G}^{\beta}(A)\subseteq \bigcap\{[x_{h_{A}(x)}]_{R}^{G}: x\in X\}$ hold.
	
	In summary,
	$\overline{R}_{G}^{\beta}(A)= \bigcap\{[x_{h_{A}(x)}]_{R}^{G}: x\in X\}$
	and
	$\{y: A(y)\leq [x_{h_{A}(x)}]_{R}^{G}(y)\}\in \mathscr{F}_{\beta}(X)$ hold for all $x\in X$ and $A\in\mathscr{F}(X)$.
\end{proof}

\begin{remark}\label{r:3.2}
	The above propositions provide the  equivalent expressions for $\underline{R}_{O}^{\beta}$ and $\overline{R}_{G}^{\beta}$ with $g_{A}$ and $h_{A}$ on arbitrary fuzzy relation. It is no longer need to consider  fuzzy granule
	$[x_{\lambda}]_{R}^{O}$ or $[x_{\lambda}]_{R}^{G}$ for all $x\in X$, which facilitates more efficient computation of the approximation operators.
	Note that the proofs of Proposition~\ref{p:3.2} and Proposition~\ref{p:3.1} are similar.
	Therefore, in the following we only give the proof of the $\underline{R}_{O}^{\beta}$, and the proof of the $\overline{R}_{G}^{\beta}$ can be derived in a similar way.
\end{remark}
\begin{proposition}\label{p:3.3}
	Let $R$ be a fuzzy relation on $X$. If overlap function $O$ and grouping function $G$ are dual w.r.t. $N$, then for all $A\in \mathscr{F}(X)$, we obtain that
	\begin{center}
		$(g_{A})^{N}=h_{A^{N}}$ and $(h_{A})^{N}=g_{A^{N}}$.
	\end{center}
	Further, we get
	\begin{center}
		$(\underline{R}_{O}^{\beta}(A))^{N}=\overline{R}_{G}^{\beta}(A^{N})$  and
		$(\overline{R}_{G}^{\beta}(A))^{N}=\underline{R}_{O}^{\beta}(A^{N})$.
	\end{center}
\end{proposition}
\begin{proof}
	If the operations $O$ and $G$ are dual w.r.t. $N$, then
	\begin{align*}
	(g_{A})^{N}(x)&=\underset{X_{i}\in\mathscr{F}_{\beta}(X)}{\bigwedge}N(g_{A}^{(i)}(x))\\
	&=\underset{X_{i}\in\mathscr{F}_{\beta}(X)}{\bigwedge}~~\underset{y\in X_{i}}{\bigvee}N(I_{O}(R(x, y), A(y)))\\
	&=\underset{X_{i}\in\mathscr{F}_{\beta}(X)}{\bigwedge}~~\underset{y\in X_{i}}{\bigvee}
	I^{G}(R^{N}(x, y), A^{N}(y))\\
	&=\underset{X_{i}\in\mathscr{F}_{\beta}(X)}{\bigwedge}h_{A^{N}}^{(i)}(x)\\
	&=h_{A^{N}}(x),
	\end{align*}
	where for all $A\in \mathscr{F}(X)$ and $x\in X$. Hence, it always holds $(g_{A})^{N}=h_{A^{N}}$.
	In a similar way, we obtain $(h_{A})^{N}=g_{A^{N}}$.
	
	For any $A\in\mathscr{F}(X)$ and $y\in X$, the following equations hold by Propositions~\ref{p:3.1} and~\ref{p:3.2}.
	
	\begin{align*}
	(\underline{R}_{O}^{\beta}(A))^{N}(y)&=\underset{x\in X}{\bigwedge}N([x_{g_{A}(x)}]_{R}^{O}(y))\\
	&=\underset{x\in X}{\bigwedge}N(O(R(x, y), g_{A}(x)))\\
	&=\underset{x\in X}{\bigwedge}G(R^{N}(x, y), (g_{A}(x))^{N})\\
	&=\underset{x\in X}{\bigwedge}G(R^{N}(x, y), h_{A^{N}}(x))\\
	&=\underset{x\in X}{\bigwedge}[x_{h_{A^{N}}(x)}]_{R}^{G}(y)\\
	&=\overline{R}_{G}^{\beta}(A^{N})(y).
	\end{align*}
	Therefore, we know that $(\underline{R}_{O}^{\beta}(A))^{N}=\overline{R}_{G}^{\beta}(A^{N})$.
	Similarly, $(\overline{R}_{G}^{\beta}(A))^{N}=\underline{R}_{O}^{\beta}(A^{N})$ holds.
\end{proof}

The comparable property, as a fundamental property between upper and lower rough approximation operator is discussed in literature \cite{Ciucci2009Approximationalgebra,Csajb2014Fromvagueness,Yao1996Twoviews}.
Next, we study several situations where $(O,G)$-GVPFRSs satisfy comparable property.
\begin{remark}\label{r:3.3}
	Based on the variable precision $\beta$, the comparable property of $O$-granular variable precision
	lower approximation operator and $G$-granular variable precision upper approximation operator are discussed below in three cases.
	
\end{remark}
\begin{itemize}
	\item (1)~Variable precision $\beta=1$\\
	In particular, when the value of $\beta$ is 1, we have $\mathscr{F}_{\beta}(X)=\{X\}$. Then for all $A\in\mathscr{F}(X)$ and $x\in X$,
	
	\begin{align*}
	g_{A}(x)&=\underset{y\in X}{\bigwedge}I_{O}(R(x, y), A(y)).
	\end{align*}
	According to Proposition~\ref{p:3.1}, we obtain that for all $z\in X$,
	\begin{align*}
	\underline{R}_{O}^{\beta}(A)(z)&=\underset{x\in X}{\bigvee}O(R(x, z), g_{A}(x))\\
	&=\underset{x\in X}{\bigvee}O(R(x, z), \underset{y\in X}{\bigwedge}I_{O}(R(x, y), A(y)))\\
	&\leq \underset{x\in X}{\bigvee}O(R(x, z), I_{O}(R(x, z), A(z)))\\
	&\leq A(z).
	\end{align*}
	Hence, if $\beta=1$, it always holds that $\underline{R}_{O}^{\beta}(A)\subseteq A$. In a similar way, $\overline{R}_{G}^{\beta}(A)\supseteq A$ can be proved.	
	
	Furthermore, let $R$ be a fuzzy $O$-similarity relation and $O$ (resp. $G$) satisfy (O6) (resp.(G6)), then by Theorem 4.1.3 in \cite{Chen2011Granularcomputing}, we can obtain that  	
	\begin{align*}
	\underline{R}_{O}^{\beta}(A)(x)&=\underset{y\in X}{\bigwedge}I_{O}(R(x, y),A(y))\ and\  \overline{R}_{G}^{\beta}(A)=\underset{y\in X}{\bigvee}I^{G}(R^{N}(x, y), A(y)),
	\end{align*}	
	for all $A\in\mathscr{F}(X)$ and $x\in X$.
	
	It follows from the reflexivity of $R$ and Lemma~\ref{l:neutal} (1) that for all $A\in\mathscr{F}(X)$ and $x\in X$,
	\begin{align*}
	\underline{R}_{O}^{\beta}(A)(x)&\le I_{O}(R(x, x),A(x)) = I_{O}(1,A(x)) = A(x),\\ \overline{R}_{G}^{\beta}(A)(x)&\ge I^{G}(R^{N}(x, x), A(x)) =I^{G}(0, A(x)) = A(x).
	\end{align*}	
	Hence, $\underline{R}_{O}^{\beta}(A)\subseteq
	A\subseteq \overline{R}_{G}^{\beta}(A)$ holds for all $A\in\mathscr{F}(X)$. As $X$ is finite, then $\underline{R}_{O}^{\beta}(A)\subseteq
	A\subseteq \overline{R}_{G}^{\beta}(A)$ holds for all $A\in\mathscr{F}(X)$ and   $\frac{|X|-1}{|X|}<\beta\le 1$.
	
	\item (2)~
	Arbitary variable precision $\beta$ and fuzzy $O$-similarity relation $R$
	
	Even if overlap function $O$ and grouping function $G$ are dual w.r.t the standard negation $N(x) = 1-x$ for all $x\in [0,1]$, $\underline{R}_{O}^{\beta}(A)$ and $\overline{R}_{G}^{\beta}(A)$ do not have comparable properties.
    A specific example is given below.
	
	Let $X=\{x_1,x_2,x_3\}$ and fuzzy relation $R$ on $X$ as
	$$
	\begin{gathered}
	R=\begin{bmatrix} 1 & 0.6 & 1 \\ 0.6 & 1 & 0.6 \\ 1 & 0.6 & 1 \end{bmatrix}
	\end{gathered}
	$$
	
	Here, we use overlap function $O$ and fuzzy implication $I_O$ defined as, respectively,
	
	$$O(x,y)=xy\ \textnormal{and} \ I_O(x,y)=
	\begin{cases}
	\frac{y}{x}\wedge 1 , & x \ne 0\\
	1, & x=0
	\end{cases}
	\ \textnormal{for all} \ x,y\in [0,1].$$
	
	 It is easy to see that fuzzy relation $R$ is a fuzzy $O$-similarity relation for overlap function $O$.
	Let  $A=\frac{0.8}{x_1} + \frac{0.1}{x_2} + \frac{0.6}{x_3}$ and $\beta=0.5$. By Theorem 2 in \cite{Yao2014Anovelvariable}, it holds that
	\begin{align*}
	\underline{R}_{O}^{\beta}(A)=g(A)=\frac{0.6}{x_1} + \frac{1}{x_2} + \frac{0.6}{x_3}.
	\end{align*}
	
	According to Theorem 3(1)in \cite{Yao2014Anovelvariable} or Proposition~\ref{p:3.3}, we can obtain
	\begin{align*}
	\overline{R}_{G}^{\beta}(A)=(\underline{R}_{O}^{\beta}(A^N))^N=\frac{0.4}{x_1} + \frac{0}{x_2} + \frac{0.4}{x_3},
	\end{align*}
	where $N$ is standard negation $N(x)=1-x$ for all $x\in [0,1]$ and grouping function $G$ takes $G(x,y)=1-(1-x)(1-y)$ for all $x,y\in [0,1]$.
	
	\item (3)~ Arbitary variable precision $\beta$ and fuzzy relation $R$
	
	Let $X=\{x_1,x_2,x_3\}$ and fuzzy relation $R$ on $X$ as
	\begin{align*}
	\begin{gathered}
	R=\begin{bmatrix} 0 & 0.2 & 0.8 \\ 1 & 0 & 1 \\ 0 & 0.1 & 0  \end{bmatrix}
	\end{gathered}.
	\end{align*}
	Since
	$\underset{y\in X}{\bigvee}\{O(R(x_1, y),R(y, x_1)\}=0.2>R(x_1,x_1)$, $R$ is not $O$-transitive. It is obvious that $R$ is not fuzzy $O$-similarity relation.
    Here, the overlap function $O$ and the fuzzy implication $I_O$ from Case(2) continue to be followed.
    Let $A=\frac{0.2}{x_1} +\frac{0}{x_2} +\frac{0.6}{x_3}$ and $\beta =0.5$, By Proposition~\ref{p:3.1}, it holds that
	\begin{align*}
	g_A = \frac{0.75}{x_1}+\frac{0.6}{x_2}+\frac{1}{x_3}.
	\end{align*}
	Furthermore, we conclude that
	\begin{align*}
	\underline{R}_{O}^{\beta}(A)=\frac{0.6}{x_1}+\frac{0.15}{x_2}+\frac{0.6}{x_3}.
	\end{align*}
	Next, we reckon the $G$-granular variable precision upper approximation operator with  $N(x)=1-x$, $G(x,y)= \textnormal{max}\{x,y\}$ and $I^G(x,y) =\left\{
	\begin{aligned}
	y, x<y \\
	0, x\ge y
	\end{aligned}
	\right.
	$ for all $x,y\in [0,1]$.
	It follows from Proposition~\ref{p:3.2} that
	\begin{align*}
	\overline{R}_{G}^{\beta}(A)=\frac{0.2}{x_1}+\frac{0.8}{x_2}+\frac{0.2}{x_3}.
	\end{align*}
	Hence, $\underline{R}_{O}^{\beta}$ and  $\overline{R}_{G}^{\beta}$ are not comparable, where  $\underline{R}_{O}^{\beta}(A)$ and $\overline{R}_{G}^{\beta}(A)$ are not dual w.r.t. the standard negation $N$.
\end{itemize}

\subsection {The degenerated $(O,G)$-granular variable precision fuzzy rough sets}
We define $[x]_R = \{y:R(x,y)=1\}$ when $R$ is a crisp relation on $X$. In particular, if fuzzy relations $R$ and fuzzy sets $A$ take crisp relations and crisp sets, we call the existing models as the degenerated $(O,G)$-GVPFRSs.

\begin{lemma}\label{l:3.1}
	Let $R$ be a crisp relation on $X$, then it holds that for all $A\in \mathscr{F}(X)$,	
	\begin{align*}
	(\underline{R}_{O}^{\beta}(A)) ^N =\overline{R}_{G}^{\beta}(A^N)\  and\  (\overline{R}_{G}^{\beta}(A))^N =\underline{R}_{O}^{\beta}(A^N).
	\end{align*}
\end{lemma}
\begin{proof}
	According to the character of crisp relation, then $[x_\lambda]^O_R = [x_\lambda]^{\land}_R\ \textnormal{and}\  [x_\lambda]^G_R = [x_\lambda]^{\vee}_R$ hold for all $x\in X$ and $\lambda \in [0,1]$. Due to the duality of minimum and maximum w.r.t. $N$ and  Proposition~\ref{p:3.3}, for all $A\in\mathscr{F}(X)$, it holds that
	\begin{align*}
	\quad (\underline{R}_{O}^{\beta}(A))^N =\overline{R}_{G}^{\beta}(A^N)\ and\  (\overline{R}_{G}^{\beta}(A))^N =\underline{R}_{O}^{\beta}(A^N).
	\end{align*}
\end{proof}
\begin{proposition}\label{p:3.4}
	Let $R$ be a crisp relation on $X$ and $A\subseteq X$ be a crisp set, then
	\begin{align*}
	\underline{R}_{O}^{\beta}(A)& = \bigcup \{[x]_R:x\in X, |[x]_R\cap A^c|\le (1-\beta)|X|\},\\\overline{R}_{G}^{\beta}(A)& =\bigcap \{[x]_R^c:x\in X, |[x]_R\cap A|\le (1-\beta)|X|\}.
	\end{align*}	
\end{proposition}
\begin{proof}
	For any $\lambda \in (0,1]$ and crisp sets $A\subseteq X$, we will prove the following holds.
	\begin{align*}
	\{y:[x]_R^O(y)\le A(y)\}=\{y:[x]_R(y)\le A(y)\}.
	\end{align*}
	
	Let $O(R(x,y),\lambda)\le A(y)$. If $A(y)=1$, it is clear that $R(x,y)\le A(y)$. If $A(y)=0,$ we can get that $R(x,y)=0$, otherwise, $O(R(x,y),\lambda)=\lambda \le 0 $,
	which contradicts with $\lambda \in (0,1]$.	Thus, $\{y:[x]_R^O(y)\le A(y)\}\subseteq \{y:[x]_R(y)\le A(y)\}$.
	
 On the other side, $ \{y:\{[x]_R^O(y)\le A(y)\}\supseteq \{y:[x]_R(y)\le A(y)\}$ can hold apparently. Hence, it always holds that $\{y:[x]_R^O(y)\le A(y)\}=\{y:[x]_R(y)\le A(y)\}$ for any crisp sets $A$. Then it follows Definition~\ref{p:3.2} that
	\begin{align*}
	\quad \underline{R}_{O}^{\beta}(A) &= \bigcup \{[x]_R^O:x\in X, \lambda \in [0,1],\{y:[x_\lambda]_R^O(y)\le A(y)\}\in \mathscr{F}_{\beta}(X)\}
	\\ &= \bigcup \{[x]_R:x\in X, \{y:[x]_R(y)\le A(y)\}\in \mathscr{F}_{\beta}(X)\}.
	\end{align*}
	
	Further, we have the following equivalences,
	\begin{align*}
	\quad \{y:[x]_R(y)\le A(y)\}\in\mathscr{F}_{\beta}(X)
	& \iff A\cup (A^c\cap [x]_R^c)\in\mathscr{F}_{\beta}(X) \\
	& \iff |A\cup (A^c\cap [x]_R^c)|\ge\beta |X| \\
	& \iff |A^c\cap [x]_R|\le(1-\beta)|X|,
	\end{align*}	
	then $\underline{R}_{O}^{\beta}(A)= \bigcup \{[x]_R:x\in X, |A^c\cap [x]_R|\le(1-\beta)|X|\}$  for all crisp sets $A$. In addition,
	$R^N=R^c$ and $A^N=A^c$ hold when $R$ and $A$ are crisp relation and crisp set.
	The other equation can be obtained by Lemma~\ref{p:3.1}.	
\end{proof}

\begin{proposition}\label{p:3.5}
	Let $R$ and $A$ be crisp relation and crisp subset on $X$, the following statements hold.
	\begin{enumerate}[(1)]
		\item
		Assuming that $R$ is reflexive, then
		\begin{align*}
		\{x:|[x]_R\cap A^c|\le (1-\beta)|X|\}\subseteq \underline{R}_{O}^{\beta}(A)\  and\
		\overline{R}_{G}^{\beta}(A)\subseteq \{x:|[x]_R\cap A|> (1-\beta)|X|\}.
		\end{align*}
		\item
		Assuming that $R$ is transitive, then
		\begin{align*}
		\underline{R}_{O}^{\beta}(A)\subseteq \{x:|[x]_R\cap A^c|\le (1-\beta)|X|\}\ and\
		\{x:|[x]_R\cap A|> (1-\beta)|X|\}\subseteq \overline{R}_{G}^{\beta}(A).
		\end{align*}
		\item
		Assuming that $R$ is a preorder relation, then
		\begin{align*}
		\underline{R}_{O}^{\beta}(A)= \{x:|[x]_R\cap A^c|\le (1-\beta)|X|\}\ and\
		\overline{R}_{G}^{\beta}(A)=\{x:|[x]_R\cap A|> (1-\beta)|X|\}.
		\end{align*}
	\end{enumerate}   	
\end{proposition}

\begin{proof}
	\begin{enumerate}[(1)]
		\item
		Since $R$ is reflexive, then it follows Proposition~\ref{p:3.4} that
		\begin{align*}
		\{x:|[x]_R\cap A^c|\le (1-\beta)|X|\}\subseteq \bigcup\{x:|[x]_R\cap A^c|\le (1-\beta)|X|\}=\underline{R}_{O}^{\beta}(A).
		\end{align*}
		
		Further according to Lemma~\ref{l:3.1}, one has that
		\begin{align*}
		\overline{R}_{G}^{\beta}(A)	=(\underline{R}_{O}^{\beta}(A^c))^c \subseteq
		\{x:|[x]_R\cap A|\le (1-\beta)|X|\}^c = \{x:|[x]_R\cap A|> (1-\beta)|X|\}.
		\end{align*}
		
		\item
		For any $w\in\underline{R}_{O}^{\beta}(A)$, there exits an $x\in X$ such that $ w\in [x]_R$ and $|[x]_R\cap A^c|\le (1-\beta)|X|$. Due to the transitivity of $R$, $R(w,y)\le R(x,y)$ holds for all $y\in X$. Therefore, we obtain $[w]_R\cap A^c \subseteq [x]_R\cap A^c$. Furthermore, it follows Proposition~\ref{p:3.4} that
		\begin{align*}
		w\in\{x:|[x]_R\cap A^c|\le (1-\beta)|X|\}.
		\end{align*}
		
		So $\underline{R}_{O}^{\beta}(A)\subseteq 	\{x:|[x]_R\cap A^c|\le (1-\beta)|X|\}$.  According to Lemma~\ref{l:3.1} that
		\begin{align*}
		\overline{R}_{G}^{\beta}(A)=(\underline{R}_{O}^{\beta}(A^c))^c\supseteq (\{x:|[x]_R\cap A|\le (1-\beta)|X|\})^c = \{x:|[x]_R\cap A|> (1-\beta)|X|\}.
		\end{align*}
		\item
		It can be proved by item (1) and item (2).
	\end{enumerate}
\end{proof}

 \section{Characterizations of the $(O,G)$-granular variable precision fuzzy rough sets}
By Remark~\ref{r:3.2}, we realise that two fuzzy sets $g_A$ and $h_A$ are vital to calculate the $\overline{R}_{G}^{\beta}$ and $\underline{R}_{O}^{\beta}$, respectively. Thus, we start this section with discussing their relevant properties. And then, some conclusions are drawn under diverse conditions.

\subsection{Some conclusions based on general fuzzy relations}

\begin{lemma}\label{l:4.1}
	Let $R$ be a fuzzy relation on $X$, then the following statements hold.
	\begin{enumerate}[(1)]
		\item 
		$g^{(i)}_{({{\cap}_{k\in I}A_k})} = \bigcap _{k\in I} g^{(i)}_{A_k} \ and \ h^{(i)}_{({{\cup}_{k\in I}A_k})} = \bigcup _{k\in I} h^{(i)}_{A_k} $ for all  $X_{i}\in \mathscr{F}_{\beta}(X)$  and  $\{A_k\}_{k \in I}\subseteq \mathscr{F}(X).$
		
		\item 
		$A\subseteq B \ implies \ g_A \subseteq g_B\ and \ h_A \subseteq h_B\ for\ all\ A,B\in \mathscr{F}(X).$
	\end{enumerate}
\end{lemma}

\begin{proof}
	\begin{enumerate}[(1)]
		\item 
		By Lemma~\ref{l:nocondition}(2), it holds that
		\begin{align*}
		g^{(i)}_{(\bigcap_{k\in I} A_k)}(x)
		&=\underset{y\in X_{i}}{\bigwedge}I_O(R(x,y),(\underset{k\in I}{\bigwedge}A_k(x))) \\
		&=\underset{y\in X_{i}}{\bigwedge}\ \underset{k\in I}{\bigwedge}I_O(R(x,y),A_k(x)) \\
		&=\underset{k\in I}{\bigwedge}\ \underset{y\in X_{i}}{\bigwedge}I_O(R(x,y),A_k(x)) \\
		&=\underset{k\in I}{\bigwedge} g^{(i)}_{A_k}(x)\\
		&=\left(\underset{k\in I}{\bigcap} g^{(i)}_{A_k}\right)(x).
		\end{align*}
		where $x\in X$ and $X_i\in\mathscr{F}_{\beta}(X)$. Hence, we get $g^{(i)}_{({{\bigcap}_{k\in I }A_k})} = \bigcap _{k\in I } g^{(i)}_{A_k}$. In a similar way, $h^{(i)}_{({{\bigcup}_{k\in I }A_k})} = \bigcup _{k\in I } h^{(i)}_{A_k}$ can be obtained for all $X_i\in\mathscr{F}_{\beta}(X)$ and $\{A_k\}_{k\in I}\subseteq \mathscr{F}(X)$.	
		\item 
		According to item (1), it can be directly proved.
		
	\end{enumerate}
\end{proof}

\begin{lemma}\label{l:4.2}
	Let $R$ be a fuzzy relation on $X$, 1 and 0 be the identity element of overlap function $O$ and grouping function $G$, respectively. Then the following statements hold.
	\begin{enumerate}[(1)]
		\item 
		$g_X=X \ and \ h_\emptyset=\emptyset.$
		
		\item 
		$\alpha_{X}\subseteq g_{\alpha_{X}}\ and \ h_{\alpha_{X}}\subseteq \alpha_{X} \ for \ all \ \alpha\in [0,1].$
		
		\item 
		$I\!f\ A=I_O{(y_{\gamma},\alpha_{X})} \ and \ \gamma = 1,\ then $
		
		$\qquad g_A(x) =\left\{
		\begin{aligned}
		&1,& \quad 0\le \beta\le \frac{|X|-1}{|X|}, \\
		&I_O{(R(x,y),\alpha)}, &\quad \frac{|X|-1}{|X|}<\beta\le 1.
		\end{aligned}
		\right. $
		
		\item 
		$I\!f\ A=y_{\alpha},\ then$
		
		$\qquad h_A(x) =\left\{
		\begin{aligned}
		&0,& \quad 0\le \beta\le \frac{|X|-1}{|X|}, \\
		&I^G{(R^N(x,y),\alpha)}, &\quad \frac{|X|-1}{|X|}<\beta\le 1.
		\end{aligned}
		\right.$
	\end{enumerate}
\end{lemma}

\begin{proof}
	\begin{enumerate}[(1)]
		\item 
		It follows Lemma~\ref{l:neutal}(1) that $I_O(\alpha,1)=1$ and $I^G(\alpha,0)=0$ for all $\alpha\in [0,1]$. Then for all $x\in X$,	
		\begin{align*}
		g_X(x) &=\underset{X_{i}\in \mathscr{F}_{\beta}(X)}{\bigvee}\ \underset{y\in X_{i}}{\bigwedge}(I_O(R(x,y),X(y)))=\underset{X_{i}\in \mathscr{F}_{\beta}(X)}{\bigvee}\ \underset{y\in X_{i}}{\bigwedge}(I_O(R(x,y),1)) =1,\\	
		h_\emptyset (x) &=\underset{X_{i}\in \mathscr{F}_{\beta}(X)}{\bigwedge}\ \underset{y\in X_{i}}{\bigvee}(I^G(R^N(x,y),\emptyset(y)))=\underset{X_{i}\in \mathscr{F}_{\beta}(X)}{\bigwedge}\ \underset{y\in X_{i}}{\bigvee}(I^G(R^N(x,y),0)) =0.
		\end{align*}
		Hence, we get that $g_X=X$ and $h_\emptyset = \emptyset.$
		
		\item 
		From Lemma~\ref{l:neutal}(3), it follows that
		\begin{align*}
		g_{\alpha_{X}}(x)
		&=\underset{X_{i}\in \mathscr{F}_{\beta}(X)}{\bigvee}\ \underset{y\in X_{i}}{\bigwedge}I_O(R(x,y),\alpha_X (x))\\
		&=\underset{X_{i}\in \mathscr{F}_{\beta}(X)}{\bigvee}\ \underset{y\in X_{i}}{\bigwedge}I_O(R(x,y),\alpha)\\
		&\ge\alpha_{X}(x).
		\end{align*}
		Hence, we get $\alpha_{X}\subseteq g_{\alpha_{X}}$. In a similar way,  $h_{\alpha_{X}}\subseteq \alpha_{X}$ holds.
		
		\item 
		It is easy to get $\mathscr{F}_{\beta}(X)=\{X\}$ when $\frac{|X|-1}{|X|} < \beta\le 1$.
		Let $A=I_O(y_\gamma,\alpha_{X})$ and $\gamma = 1$, for any $x\in X$,
		\begin{align*}
		g_A(x)
		&=\underset{z\in X}{\bigwedge}I_O(R(x,z),A(z))\\
		&=I_O(R(x,y),I_O(y_\gamma (y),\alpha_{X}(y)))\\
		&=I_O(R(x,y),I_O(1,\alpha))\\
		&=I_O(R(x,y),\alpha).
		\end{align*} 	
		Otherwise, if $0\le\beta\le\frac{|X|-1}{|X|}$, we obtain $X-\{y\}\in\mathscr{F}_{\beta}(X)$. Then for all $x\in X$,
		\begin{align*}
		g_A(x)&\ge \underset{z\in\, X-\{y\}}{\bigwedge} I_O(R(x,z),A(z))\\
		&= \underset{z\in \,X-\{y\}}{\bigwedge} I_O(R(x,z),I_O(y_\gamma(z),\alpha))\\
		&=1.
		\end{align*}
		\item
		The proof is similar as item (3).
	\end{enumerate}
\end{proof}

\begin{lemma}\label{l:4.3}
	Let $R$ be a fuzzy relation on $X$, overlap function $O$ and grouping function $G$ satisfy \textnormal{(O6)} and \textnormal{(G6)}, respectively. Then the following statements hold.
	\begin{enumerate}[(1)]
		\item 
		$g_{(I_O(\alpha_{X},\,A))}=I_O(\alpha_{X}, g_A)\ and \  h_{(I^G(\alpha_{X},\,A))}=I^G(\alpha_{X},h_A) \ for\ all \ \alpha\in [0,1] \ and \ A\in \mathscr{F}(X).$
		
		\item 
		$O(\alpha_{X},g_A)\subseteq g_{(O(\alpha_{X},A))}\ and \ G(\alpha_{X},h_A)\supseteq h_{(G(\alpha_{X},A))} \ for \ all \ \alpha\in [0,1] \ and \ A\in \mathscr{F}(X).$
	\end{enumerate}
\end{lemma}

\begin{proof}
	\begin{enumerate}[(1)]
		\item 
		According to Lemma~\ref{l:nocondition}(2) and Lemma~\ref{l:nocondition}(2), one has that
		\begin{align*}
		g^{(i)}_{(I_O(\alpha_{X},\,A)}(x)
		&=\underset{y\in X_{i}}{\bigwedge}I_O(R(x,y),I_O(\alpha,A(y)))\\
		&=\underset{y\in X_{i}}{\bigwedge}I_O(O(\alpha,R(x,y)),A(y))\\
		&=\underset{y\in X_{i}}{\bigwedge}I_O(\alpha,I_O(R(x,y),A(y)))\\
		&=I_O(\alpha,\underset{y\in X_{i}}{\bigwedge}(I_O(R(x,y),A(y))))\\
		&=I_O(\alpha,g^{(i)}_A(x))\\	
		&=I_O(\alpha_X,g^{(i)}_A)(x).
		\end{align*}
		where $x\in X, \alpha\in [0,1]$ and $X_i\in\mathscr{F}_{\beta}(X)$. Since $\mathscr{F}_{\beta}(X)$ is finite, by Lemma~\ref{l:nocondition}(3), then for all $A\in \mathscr{F}(X)$,
		\begin{align*}
		g_{(I_O(\alpha_{X},\,A)}(x)
		&=\underset{X_{i}\in \mathscr{F}_{\beta}(X)}{\bigvee} g^{(i)}_{(I_O(\alpha_{X},\ A))}(x)\\
		&=\underset{X_{i}\in \mathscr{F}_{\beta}(X)}{\bigvee}I_O(\alpha,g^{(i)}_{A}(x))\\
		&=I_O(\alpha,\underset{X_{i}\in \mathscr{F}_{\beta}(X)}{\bigvee}g^{(i)}_{A}(x))\\
		&=I_O(\alpha,g_{A}(x))\\
		&=I_O(\alpha_{X},g_{A})(x).
		\end{align*}
		Hence ,we get $g_{I_O(\alpha_{X},\,A)} = I_O(\alpha_{X},g_A)$.
		In a similar way,   $h_{(I^G(\alpha_{X},\,A))}=I^G(\alpha_{X},h_A)$ holds .
		
		\item 
		According to Lemma~\ref{l:associative law}(1) that for any $x\in X$,
		\begin{align*}
		g_{(O(\alpha_{X},\,A))}(x)
		&= \underset{X_{i}\in \mathscr{F}_{\beta}(X)}{\bigvee}\ \underset{y\in X_{i}}{\bigwedge}I_O(R(x,y),O(\alpha,A(y)))\\
		&\ge\underset{X_{i}\in \mathscr{F}_{\beta}(X)}{\bigvee}\ \underset{y\in X_{i}}{\bigwedge}O(\alpha,I_O(R(x,y),A(y)))\\
		&= O(\alpha,\underset{X_{i}\in \mathscr{F}_{\beta}(X)}{\bigvee}\ \underset{y\in X_{i}}{\bigwedge}I_O(R(x,y),A(y)))\\
		&=O(\alpha_X,g_A)(x).
		\end{align*}
		Then we conclude $O(\alpha_{X},g_A)\subseteq g_{(O(\alpha_{X},\,A))}$.
		In a similar way,  $G(\alpha_{X},h_A)\supseteq h_{(G(\alpha_{X},A))}$ holds.
	\end{enumerate}
\end{proof}

\begin{proposition}\label{p:4.1}
	Let $R$ be a fuzzy relation on $X$. Then the following statements hold.
	\begin{enumerate}[(1)]
		\item 
		$A\subseteq B \ implies \ \underline{R}_{O}^{\beta}(A) \subseteq \underline{R}_{O}^{\beta}(B)\ and \ \overline{R}_{G}^{\beta}(A) \subseteq \overline{R}_{G}^{\beta}(B)\ for\ all\ A,B\in \mathscr{F}(X).$
		
		\item 
		$If\ \beta>0.5,\ then\ for\ all\ A, B\in\mathscr{F}(X),$	
		\begin{align*}
		\underline{R}_{O}^{\beta}(A)\cup\underline{R}_{O}^{\beta}(B)\subseteq\underline{R}_{O}^{2\beta -1}(A\cup B)\ and\ \overline{R}_{G}^{2\beta -1}(A\cap B)\subseteq\overline{R}_{G}^{\beta}(A)\cap\overline{R}_{G}^{\beta}(B).
		\end{align*}
	\end{enumerate}
\end{proposition}

\begin{proof}
	\begin{enumerate}[(1)]
		\item 
		According to Proposition~\ref{p:3.1}, \ref{p:3.2} and Lemma~\ref{l:4.1}(2), it can be directly proved.	
		\item 
		According to Proposition~\ref{p:3.1}, it always holds that for all $A,B\in\mathscr{F}{(X)}$,
		\begin{align*}
		\underline{R}_{O}^{\beta}(A)\cup\underline{R}_{O}^{\beta}(B)=\bigcup\{[x_{g_A(x)}]^{O}_R\cup [x_{g_B(x)}]^{O}_R:x\in X\}.
		\end{align*}
		 there exist $X_i\, ,\,X_j\in\mathscr{F}_{\beta}(X)$ such that $g_A(x)=g_A^{(i)}(x)$ and $g_B(x)=g_B^{(i)}(x)$. Hence, we have $|X_i\cap X_j|\ge(2\beta-1)|X|,\ i.e., X_i\cap X_j\in\mathscr{F}_{(2\beta-1)}(X)$. Then for any $y\in X_i\cap X_j$,
		\begin{align*}
		([x_{g_A(x)}]^{O}_R\cup [x_{g_B(x)}]^{O}_R)(y)=(O(R(x,y),g_A^{(i)}(x)))\vee (O(R(x,y),g_B^{(i)}(x)))\le A(y)\vee B(y).
		\end{align*}
		So we obtain $\underline{R}_{O}^{\beta}(A)\cup\underline{R}_{O}^{\beta}(B)\subseteq \underline{R}_{O}^{(2\beta-1)}(A\cup B)$. In a similar way, $\overline{R}_{G}^{2\beta -1}(A\cap B)\subseteq\overline{R}_{G}^{\beta}(A)\cap\overline{R}_{G}^{\beta}(B)$ holds.		
	\end{enumerate}
\end{proof}

\begin{proposition}\label{p:4.2}
	Let $R$ be a fuzzy relation on $X$, 1 and 0 be the identity element of overlap function $O$ and grouping function $G$, respectively. Then the following statements hold.
	\begin{enumerate}[(1)]
		\item  
		$O(\alpha ,\bigvee_{x\in X}R(x,z))\le\underline{R}_{O}^{\beta}(\alpha_{X})(z)
		\ and \ \overline{R}_{G}^{\beta}(\alpha_{X})(z)\le G(\alpha ,\bigwedge_{x\in X}R^N(x,z))$ for all $\alpha\in [0,1]$ and $ z\in X.$
		\item 
		$I\!f \ a\ crisp\ set\ Y\subseteq X\ and\ \beta=\frac{|Y|}{|X|},then\ for\ all\ y\in X,$
		\begin{align*}
		\underline{R}_{O}^{\beta}(Y)(z) \ge  \underset{x\in X}{\bigvee}R(x,z)\ and\ \overline{R}_{G}^{\beta}(Y^c)(z)\le \underset{x\in X}{\bigwedge}R^N(x,z).	
		\end{align*}
		\item 
		$I\!f\ A=I_O(y_\gamma,\alpha_{X})\ and\ \gamma = 1, then$
		\begin{align*}
		\underline{R}_{O}^{\beta}(A)(z) =\left\{
		\begin{aligned}
		&\bigvee_{x\in X}R(x,z),& \quad 0\le \beta\le \frac{|X|-1}{|X|}, \\
		&\bigvee_{x\in X}O(R(x,z),I_O(R(x,y),\alpha)), &\quad \frac{|X|-1}{|X|}<\beta\le 1.
		\end{aligned}
		\right.
		\end{align*}
		
		\item 
		$If\ A=y_\alpha,then$
		\begin{align*}
		\overline{R}_{G}^{\beta}(A)(z) =\left\{
		\begin{aligned}
		&\bigwedge_{x\in X}R^N(x,z),& \quad 0\le \beta\le \frac{|X|-1}{|X|}, \\
		&\bigwedge_{x\in X}(G(R^N(x,z),I^G(R^N(x,y),\alpha))), &\quad \frac{|X|-1}{|X|}<\beta\le 1.
		\end{aligned}
		\right.
		\end{align*}
	\end{enumerate}
\end{proposition}

\begin{proof}
	\begin{enumerate}[(1)]	
		\item 
		According to Lemma~\ref{l:4.2}(2), it can be directly proved.	
		\item
		If $\beta=\frac{|Y|}{|X|}$, then $Y\in\mathscr{F}_{\beta}(X) $ and for all $x\in X$,
		\begin{align*}
		g_Y(x)\ge\underset{y\in Y}{\bigwedge}I_O(R(x,y),Y(y))=1.
		\end{align*}
		Furthermore, for any $z\in X$, 	
		\begin{align*}
		\underline{R}_{O}^{\beta}(Y)(z)
		&=\underset{x\in X}{\bigvee}O(R(x,z),g_Y(x))\\
		&=\underset{x\in X}{\bigvee}O(R(x,z),1)\\
		&= \underset{x\in X}{\bigvee}R(x,z).
		\end{align*}
		\item
		According to Lemma~\ref{l:4.2}(3), it can be directly proved.	
		\item
		According to Lemma~\ref{l:4.2}(4), it can be directly proved. 		
	\end{enumerate}
\end{proof}



\begin{remark}\label{r:4.1}
	Consider $X=\{x_1,x_2,x_3\}$ and the fuzzy relation $R$ on $X$ as
	\begin{align*}
	\begin{gathered}
	R=\begin{bmatrix} 0 & 0.4 & 0.4 \\ 0.2 & 0 & 0.2 \\ 0.2 & 0.2 & 0 \end{bmatrix}
	\end{gathered}
	\end{align*}
	Here, we use overlap function $O_{DB}$ and fuzzy implication $I_O$
	defined as, respectively,
	
	\begin{align*}
	O_{DB}(x,y) =\left\{
	\begin{aligned}
	\frac{2xy}{x+y}\,&,\quad if\ x+y\ne 0 \\
	0\,&,\quad if\ x+y= 0
	\end{aligned}
	\right. 	
	\qquad I_O(x,y) =\left\{
	\begin{aligned}
	\frac{xy}{2x-y}\,&,\quad if\ y< \frac{2x}{x+1}\\
	1\,&,\quad if\ y\ge \frac{2x}{x+1}
	\end{aligned}
	\right.
	\end{align*}
	for all $x,y\in [0,1].$ Let  $A=\frac{0.2}{x_1}+\frac{0.4}{x_2}+\frac{0}{x_3}, \alpha = 1$ and $\beta = 0.5.$ Then from Proposition~\ref{p:3.1} that
	\begin{align*}
	\underline{R}_{O}^{\beta}(O(\alpha_{X},A)) = \frac{\frac{1}{3}}{x_1}+\frac{\frac{4}{7}}{x_2}+\frac{\frac{4}{7}}{x_3},
	\end{align*}
	and
	\begin{align*}
	O(\alpha_{X},\underline{R}_{O}^{\beta}(A)) = \frac{\frac{1}{2}}{x_1}+\frac{\frac{4}{7}}{x_2}+\frac{\frac{4}{7}}{x_3}.
	\end{align*}
	
	By comparison, we get 	$O(\alpha_{X},\underline{R}_{O}^{\beta}(A))\supseteq \underline{R}_{O}^{\beta}(O(\alpha_{X},A)).$ In this example,  the overlap function $O_{DB}$ does not satisfy the associative law. Furthermore, according to the above conditions, we get $\underline{R}_{O}^{\beta}(A)=\frac{\frac{1}{3}}{x_1}+\frac{\frac{2}{5}}{x_2}+\frac{\frac{2}{5}}{x_3}.$ In particular, we take $\alpha = 1$. It follows from Proposition~\ref{p:3.1} that
	\begin{align*}
	I_O(\alpha_{X},\underline{R}_{O}^{\beta}(A))=\frac{\frac{1}{5}}{x_1}+\frac{\frac{1}{4}}{x_2}+\frac{\frac{1}{4}}{x_3},
	\end{align*}
	and
	\begin{align*}
	\underline{R}_{O}^{\beta}(I_O(\alpha_{X},A))= \frac{\frac{1}{4}}{x_1}+\frac{\frac{1}{4}}{x_2}+\frac{\frac{1}{4}}{x_3}.
	\end{align*}
	
	Hence, $\underline{R}_{O}^{\beta}(I_O(\alpha_{X},A))\supseteq I_O(\alpha_{X},\underline{R}_{O}^{\beta}(A)).$
	
	In particular, the following conclusions can be given when overlap and grouping functions satisfy the associative law.	
\end{remark}

\begin{proposition}\label{p:4.3}
	Let $R$ be a fuzzy relation on $X$, overlap function $O$ and grouping function $G$ satisfy \textnormal{(O6)} and \textnormal{(G6)}, respectively. For all $\alpha\in [0,1]$ and $ A\in\mathscr{F}(X)$, the following statements hold.
	
	\begin{enumerate}[(1)]
		\item 
		$\underline{R}_{O}^{\beta}(I_O(\alpha_{X},\ A))\subseteq I_O(\alpha_{X},\underline{R}_{O}^{\beta}(A)) \ and \ I^G(\alpha_{X},\overline{R}_{G}^{\beta}(A))\subseteq \overline{R}_{G}^{\beta}(I^G(\alpha_{X},\ A))$. Especially,	
		\begin{align*}
		&\underline{R}_{O}^{\beta}(I_O(\alpha_{X},\emptyset))=I_O(\alpha_{X},\emptyset)\ implies \ \underline{R}_{O}^{\beta}(\emptyset) =\emptyset;\\
		&\overline{R}_{G}^{\beta}(I^G(\alpha_{X},X))=I^G(\alpha_{X},X)\ implies \ \overline{R}_{G}^{\beta}(X) =X.
		\end{align*}
		
		\item
		$O(\alpha_{X},\underline{R}_{O}^{\beta}(A)\subseteq \underline{R}_{O}^{\beta}(O(\alpha_{X},A))$ and $
		\overline{R}_{G}^{\beta}(G(\alpha_{X},A))\subseteq G(\alpha_{X},\overline{R}_{G}^{\beta}(A))$.
	\end{enumerate}
\end{proposition}

\begin{proof}
	\begin{enumerate}[(1)]
		\item 
		By Lemma~\ref{l:4.3}(1), \ref{l:nocondition}(3) and \ref{l:associative law}(1), it holds that for all $z\in X$,
		\begin{align*}
		\underline{R}_{O}^{\beta}(I_O(\alpha_{X},\ A))(z)
		&=\underset{x\in X}{\bigvee}O(R(x,z),g_{(I_O(\alpha_{X},\,A))}(x))\\
		&=\underset{x\in X}{\bigvee}O(R(x,z),I_O(\alpha,g_A(x)) \\
		&\le\underset{x\in X}{\bigvee}I_O(\alpha,O(R(x,z),g_A(x)))\\
		&=I_O(\alpha,\underset{x\in X}{\bigvee}O(R(x,z),g_A(x)))\\
		&=I_O(\alpha,\underline{R}_{O}^{\beta}(A)(z)).	
		\end{align*}
		Hence, we get $\underline{R}_{O}^{\beta}(I_O(\alpha_{X},A))\subseteq I_O(\alpha_{X},\underline{R}_{O}^{\beta}(A))$.
		In a similar way,  $I^G(\alpha_{X},\overline{R}_{G}^{\beta}(A))\subseteq \overline{R}_{G}^{\beta}(I^G(\alpha_{X},A))$ holds.	
		\item
		It follows from Lemma~\ref{l:4.3}(2) and the associativity of the overlap function that
		\begin{align*}
		\underline{R}_{O}^{\beta}(O(\alpha_{X},A))(z)
		& = \bigvee_{x\in X}O(R(x,z),g_{(O(\alpha_{X}, g_A(x)))})\\
		& \ge \bigvee_{x\in X}O(R(x,z),O(\alpha,g_A(x))\\
		& = \bigvee_{x\in X}O(\alpha,O(R(x,z),g_A(x)))\\
		& = O(\alpha, \bigvee_{x\in X}O(R(x,z),g_A(x)))\\
		& = O(\alpha_{X},\underline{R}_{O}^{\beta}(A))(z).
		\end{align*} 	
		Hence, we get $\underline{R}_{O}^{\beta}(O(\alpha_{X},A))\supseteq O(\alpha_{X},\underline{R}_{O}^{\beta}(A)).$
		In a similar way,  $\overline{R}_{G}^{\beta}(G(\alpha_{X},A))\subseteq G(\alpha_{X},\overline{R}_{G}^{\beta}(A))$ holds for all $\alpha\in [0,1]$ and $A\in \mathscr{F}(X).$	 		
	\end{enumerate}
\end{proof}
\subsection{Some new conclusions based on special fuzzy relations}
\begin{proposition}\label{p:serial&identity}
	Let $R$ be a fuzzy relation on $X$, 1 and 0 be the identity element of overlap function $O$ and grouping function $G$, respectively. For all $\alpha\in [0,1]$, the following statements are equivalent.
	\begin{enumerate}[(1)]
		\item 
		$R^{-1}\ is\ serial.$
		
		\item 
		$X=\underline{R}_{O}^{\beta}(X)$.
		
		\item 
		$\emptyset=\overline{R}_{G}^{\beta}(\emptyset)$.
		
		\item 
		$\alpha_{X}\subseteq \underline{R}_{O}^{\beta}(\alpha_{X}).$
		
		\item 
		$\overline{R}_{G}^{\beta}(\alpha_{X})\subseteq \alpha_{X}.$
		
		\item 
		$A\ crisp\ set\ Y\subseteq X\ and\ \beta=\frac{|Y|}{|X|}\ imply\ \underline{R}_{O}^{\beta}(Y)=X.$
		
		\item 
		$A\ crisp\ set\ Y\subseteq X\ and\ \beta=\frac{|Y|}{|X|}\ imply\ \overline{R}_{G}^{\beta}(Y^c)=\emptyset.$
		
		\item 
	$\underline{R}_{O}^{\beta}(I_O(y_\gamma,\alpha_{X}))=X\ f\!or\ all\ y\in X$, if $\beta\le\frac{|X|-1}{|X|}\ and\ \gamma=1$.
		
		\item 
		$\overline{R}_{G}^{\beta}(y_\alpha)=
		\emptyset\ f\!or\ all\ y\in X$, if $\beta\le\frac{|X|-1}{|X|}$.
	\end{enumerate}
\end{proposition}

\begin{proof}
	By Proposition~\ref{p:4.2}(2), (3) and (4), it holds that
	
	(1)$\iff$(6)$\iff$(7)$\iff$(8)$\iff$(9).
	
	Furthermore, we can obtain (1)$\Rightarrow$(4)$\Rightarrow$(2) by Proposition~\ref{p:4.2}(3). The next will prove (2)$\Rightarrow$(1).
	Suppose $R^{-1}$ is not serial, then the existence of $z_0\in X$ leads to $\bigvee_{x\in X}R(x,z_0)<1.$ Further, from Lemma~\ref{l:4.2}(1), one concludes that
	\begin{align*}
	\underline{R}_{O}^{\beta}(X)(z_0)&=\underset{x\in X}{\bigvee}O(R(x,z_0),g_X(x))\\
	&=\underset{x\in X}{\bigvee}O(R(x,z_0),1)\\
	&=O(\underset{x\in X}{\bigvee}R(x,z_0),1)\\
	&<O(1,1)=1,
	\end{align*}
	which contradicts with $X=\underline{R}_{O}^{\beta}(X)$. Therefore $R^{-1}$ is serial. In a similar way, we obtain (1)$\Rightarrow$(5)$\Rightarrow$(3)$\Rightarrow$(1).		
\end{proof}
Furthermore, the equivalent conditions related to the seriality property of $R$ can be easily obtained by exchanging the positions of $R$ and $R^{-1}$.

\begin{proposition}\label{p:4.5}
	Let $R$ satisfy reflexivity, then the following statements hold.
	\begin{enumerate}[(1)]
		\item
		$g_A\subseteq  \underline{R}_{O}^{\beta}(A)\ and\ \overline{R}_{G}^{\beta}(A)\subseteq h_A\ for\ all\ A\in\mathscr{F}(X).$
		
		\item
		$\underline{R}_{O}^{\beta}(X)=X\ and\ \overline{R}_{G}^{\beta}(\emptyset)=\emptyset.$
		
		\item
		$\alpha_{X}\subseteq\underline{R}_{O}^{\beta}(\alpha_{X})\ and\ \overline{R}_{G}^{\beta}(\alpha_{X})\subseteq\alpha_{X}\ for\ all\ \alpha\in [0,1].$
		
		\item
		
		$\underline{R}_{O}^{\beta}(Y)=X\ and\  \overline{R}_{G}^{\beta}(Y^c)=\emptyset,$ where $Y$ is a crisp set on $X$ and $\beta=\frac{|Y|}{|X|}$.
		\item
		$\underline{R}_{O}^{\beta}(I_O(y_\gamma,\alpha_{X}))=X\ and\ \overline{R}_{G}^{\beta}(y_\alpha)=\emptyset $, if $\beta\le\frac{|X|-1}{|X|}, \gamma=1 \ and\ \alpha\in [0,1].$	
	\end{enumerate} 	
\end{proposition}

\begin{proof}
	According to Proposition~\ref{p:3.1}, \ref{p:3.2} and \ref{p:serial&identity}, it can be directly proven.
\end{proof}

\begin{proposition}\label{p:4.6}
	Let $R$ satify symmetry. For all $A\in\mathscr{F}(X)$, the following statements hold.
	\begin{align*}
	\underline{R}_{O}^{\beta}(A)=\underline{(R^{-1})}_{O}^{\beta}(A)\ and\ \overline{R}_{G}^{\beta}(A)=\overline{(R^{-1})}_{G}^{\beta}(A).
	\end{align*}
\end{proposition}

\begin{proof}
	It can be easily derived from the symmetry of $R$.
\end{proof}

\begin{remark}\label{r:4.2}
	Consider $X=\{x_1,x_2,x_3\}$ and crisp relation $R$ on $X$ as
	\begin{align*}
	\begin{gathered}
	R=\begin{bmatrix} 1 & 0 & 1 \\ 0 & 1 & 0 \\ 1 & 0 & 1 \end{bmatrix}
	\end{gathered}
	\end{align*}
	
	It is easy to conclude that crisp relation $R$ is a fuzzy $O$-similarity relation for any overlap function $O$.
	
	Here, we apply overlap function $ O_{DB}$ and its residual implication $I_O$  defined as Remark~\ref{r:4.1}. Let $A=\frac{0.2}{x_1}+\frac{0}{x_2}+\frac{0.5}{x_3},\ \alpha = \bigwedge_{x\in X}R(x,x)$ and $\beta  = 0.5$. By Proposition~\ref{p:3.1}, we conclude that
	\begin{align*}
	g_A=\frac{\frac{1}{3}}{x_1}+\frac{1}{x_2}+\frac{\frac{1}{3}}{x_3}.
	\end{align*}
	Hence, the $O$-granular variable precision lower approximation operator is
	\begin{align*}
	\underline{R}_{O}^{\beta}(A)=\frac{\frac{1}{2}}{x_1}+\frac{1}{x_2}+\frac{\frac{1}{2}}{x_3}.
	\end{align*}
	Furthermore, we obtain the following conclusions,
	\begin{align*}
	O(\alpha_{X},\underline{R}_{O}^{\beta}(A))=\frac{\frac{2}{3}}{x_1}+\frac{1}{x_2}+\frac{\frac{2}{3}}{x_3},
	\end{align*}
	and
	\begin{align*}
	\underline{R}_{O}^{\beta}(\underline{R}_{O}^{\beta}(A))=\frac{\frac{1}{2}}{x_1}+\frac{1}{x_2}+\frac{\frac{1}{2}}{x_3}.
	\end{align*}
	It indicates that $g_A\subseteq\underline{R}_{O}^{\beta}(A)$ and $O(\alpha_{X},\underline{R}_{O}^{\beta}(A))\supseteq\underline{R}_{O}^{\beta}(\underline{R}_{O}^{\beta}(A)).$
	However, the following conclusions can be obtained when $O$ satisfies (O6).
\end{remark}

\begin{proposition}\label{p:4.7}
	Let $R$ satisfy $O$-transitivity, $\alpha=\bigwedge_{x\in X}R(x,x)$, overlap function $O$ and grouping function $G$ satisfy \textnormal{(O6)} and \textnormal{(G6)}, respectively. For all $A\in\mathscr{F}(X)$, the following statements hold. 	
	\begin{enumerate}[(1)]
		\item
		$\underline{R}_{O}^{\beta}(A)\subseteq g_A\ and\ O(\alpha_{X},\underline{R}_{O}^{\beta}(A))\subseteq\underline{R}_{O}^{\beta}(\underline{R}_{O}^{\beta}(A)).$
		\item
		$h_A\subseteq\overline{R}_{G}^{\beta}(A)\ and\ \overline{R}_{G}^{\beta}(\overline{R}_{G}^{\beta}(A))\subseteq G((\alpha_{X})^N,\overline{R}_{G}^{\beta}(A)), \ i\!f\ O\ and\ G\ are\ dual\ w.r.t.\ N.$
	\end{enumerate}
\end{proposition}

\begin{proof}
	\begin{enumerate}[(1)]
		\item
		For any $z\in X$, there exist $x\in X$ and $X_i\in\mathscr{F}_{\beta}(X)$ such that
		\begin{align*}
		\underline{R}_{O}^{\beta}(A)(z)=O(R(x,z),g_A^{(i)}(x)).
		\end{align*}
		Furthermore, according to Lemma~\ref{l:nocondition}(1) and (O6), one concludes that for all $y\in X_i$,
		\begin{align*}	       	
		O(R(z,y),	\underline{R}_{O}^{\beta}(A)(z))
		& =O(R(z,y),O(R(x,z),g_A^{(i)}(x)))\\
		&  =O(O(R(z,y),R(x,z)),g_A^{(i)}(x))\\
		& \le O(R(x,y),g_A^{(i)}(x))\\
		& \le O(R(x,y),I_O(R(x,y),A(y)))\\
		& \le A(y).
		\end{align*}  	
		Hence, we know that $\underline{R}_{O}^{\beta}(A)(z)\le g_A^{(i)}(z)\le g_A(z)$, that is to say, $\underline{R}_{O}^{\beta}(A)\subseteq g_A$ holds.
		
		Let $B=\underline{R}_{O}^{\beta}(A).$ By Proposition~\ref{p:3.1}, then it follows that for all $y\in X$ and $X_{i}\in\mathscr{F}_{\beta}{(X)},$
		\begin{align*}
		B(y)&=\underline{R}_{O}^{\beta}(A)(y)\\
		&=\underset{x\in X}{\bigvee}O(R(x,y),g_A(x))\\	
		&\ge O(R(y,y),g_A(y))\\
		&=O(R(y,y),\underset{X_i\in\mathscr{F}(X)}{\bigvee}g_A^{(i)}(y))\\
		&\ge O(R(y,y),g^{(i)}_A(y)).
		\end{align*}
		According to Lemma~\ref{l:associative law}(1) and (2), the following holds for all $x\in X$ and $X_i\in\mathscr{F}_{\beta}(X),$
		\begin{align*}
		g^{(i)}_B(x)&=\underset{y\in X_i}{\bigwedge}I_O(R(x,y),B(y))\\
		&\ge\underset{y\in X_i}{\bigwedge}I_O(R(x,y),O(R(y,y),g_A^{(i)}(y)))\\
		&\ge \underset{y\in X_i}{\bigwedge}O(R(y,y),I_O(R(x,y),g^{(i)}_A(y)))\\
		&=\underset{y\in X_i}{\bigwedge}O(R(y,y),\underset{z\in X_i}{\bigwedge}I_O(R(x,y),I_O(R(y,z),A(z)))\\
		&= \underset{y\in X_i}{\bigwedge}O(R(y,y),\underset{z\in X_i}{\bigwedge}I_O(O(R(x,y),R(y,z)),A(z)))\\
		&\ge \underset{y\in X_i}{\bigwedge}O(R(y,y),\underset{z\in X_i}{\bigwedge}I_O(O(R(x,z),A(z))))\\
		&= \underset{y\in X_i}{\bigwedge}O(R(y,y),g^{(i)}_A(x))\\
		&= O(\underset{y\in X_i}{\bigwedge}R(y,y),g^{(i)}_A(x))\\
		&\ge O(\alpha, g^{(i)}_A(x)).
		\end{align*}
		Further, $g_B(x)\ge O(\alpha, g_A^{(i)}(x))$ can be derived, then
		\begin{align*}
		O(R(x,y),g_{\underline{R}_{O}^{\beta}(A)(x)})\ge  O(R(x,y),O(\alpha,g^{(i)}_A(x)))=O(\alpha, O(R(x,y),g^{(i)}_A(x))).
		\end{align*}
		So we obtain
		\begin{align*}
		\underline{R}_{O}^{\beta}(\underline{R}_{O}^{\beta}(A))\supseteq O(\alpha_{X},\underline{R}_{O}^{\beta}(A)).
		\end{align*}
		\item
		According to item (1) and Proposition~\ref{p:3.3}, it can be directly proved.	
	\end{enumerate}
\end{proof}

When $R$ takes the fuzzy $O$-preorder relation, we get the following conclusions.
\begin{proposition}\label{p:4.8}
	Let $R$ satisfy fuzzy $O$-preorder relation, overlap funtion $O$ and grouping function $G$ satisfy \textnormal{(O6)} and \textnormal{(G6)}, respectively. For any $A\in \mathscr{F}(X)$, the following statements hold.
	\begin{enumerate}[(1)]
		\item
		$\underline{R}_{O}^{\beta}(A)=g_A\ and\ \underline{R}_{O}^{\beta}(A)\subseteq\underline{R}_{O}^{\beta}(\underline{R}_{O}^{\beta}(A)).$
		
		\item
		$\overline{R}_{G}^{\beta}(A)=h_A\ and\ \overline{R}_{G}^{\beta}(\overline{R}_{G}^{\beta}(A))\subseteq\overline{R}_{G}^{\beta}(A),\ i\!f\ O\ and\ G\ are\ dual\ w.r.t.\ N.$
	\end{enumerate}
\end{proposition}

\begin{proof}
	According to Proposition~\ref{p:3.3}, \ref{p:4.5} and \ref{p:4.7}, it can be directly proven.
\end{proof}

\begin{proposition}\label{p:4.9}	
	Let $R$ satisfy fuzzy $O$-preorder relation, overlap funtion $O$ and grouping function $G$
	satisfy \textnormal{(O6)} and \textnormal{(G6)}, respectively. If $O$ and $G$ are dual w.r.t. $N$, then the following statements hold.
	\begin{enumerate}[(1)]
		\item
		$\underline{R}_{O}^{\beta}(I_O(\alpha_{X},A))=I_O(\alpha_{X},\underline{R}_{O}^{\beta}(A))\ and\ \overline{R}_{G}^{\beta}(I^G(\alpha_{X},A))=I^G(\alpha_{X},\overline{R}_{G}^{\beta}(A))$ for all $\alpha\in [0,1]$ and $A\in\mathscr{F}(X).$
		
		\item
		$\underline{R}_{O}^{\beta}(\emptyset)=\emptyset\ if\ and\ only\ if\ \underline{R}_{O}^{\beta}(I_O(\alpha_{X},\emptyset))=I_O(\alpha_{X},\emptyset)\ for\ all\ \alpha\in [0,1]$.
		
		\item
		$\overline{R}_{G}^{\beta}(X)=X\ if\ and\ only\ if\ \overline{R}_{G}^{\beta}(I^G(\alpha_{X},X))=I^G(\alpha_{X},X)\ for\ all\ \alpha\in [0,1]$.
		
		\item
		$I\!f\ \beta>0.5,\ then\ for\ all\ A,B\ \in\mathscr{F}(X)$,
		\begin{align*}
		\underline{R}_{O}^{\beta}(A)\cap\underline{R}_{O}^{\beta}(B)\subseteq\underline{R}_{O}^{(2\beta-1)}(A\cap B),\quad\overline{R}_{G}^{(2\beta-1)}(A\cap B)\subseteq\overline{R}_{G}^{\beta}(A)\cap\overline{R}_{G}^{\beta}(B);\\
		\underline{R}_{O}^{\beta}(A)\cup\underline{R}_{O}^{\beta}(B)\subseteq\underline{R}_{O}^{(2\beta-1)}(A\cup B),\quad\overline{R}_{G}^{(2\beta-1)}(A\cup B)\subseteq\overline{R}_{G}^{\beta}(A)\cup\overline{R}_{G}^{\beta}(B).
		\end{align*}
	\end{enumerate}
\end{proposition}

\begin{proof}\label{pr:4.8}
	\begin{enumerate}[(1)]
		\item
		Let $x\in X$ and $\lambda=I_O(\alpha,\underline{R}_{O}^{\beta}(A)(x))$, then we get $O(\alpha,\lambda)\le \underline{R}_{O}^{\beta}(A)(x).$ According to Proposition~\ref{p:3.1} and \ref{p:4.8}(1), there exists an $X_i\in\mathscr{F}_{\beta}(X)$ such that
		\begin{align*}
		O(\alpha,\lambda)\le\underline{R}_{O}^{\beta}(A)(x)=g_A(x)=g_A^{(i)}(x)=\underset{y\in X_i}{\bigwedge}I_O(R(x,y),A(y)).
		\end{align*}
		Then for all $y\in X_i$,
		\begin{align*}
		O(\alpha,\lambda)\le I_O(R(x,y),A(y))
		&\iff 	O(O(\alpha,\lambda),R(x,y))\le A(y)\\
		&\iff O(\alpha,O(\lambda,R(x,y)))\le A(y)\\
		&\iff [x_\lambda]^{O}_R(y) \le I_O(\alpha,A(y)).
		\end{align*}
	 According to Definition~\ref{d:model}, it holds that $[x_\lambda]^{O}_R\subseteq\underline{R}_{O}^{\beta}(I_O(\alpha_{X},A))$.
	 On the other hand, $\lambda = [x_\lambda]^{O}_R(x)\le\underline{R}_{O}^{\beta}(I_O(\alpha_{X},A))(x)$, since $R$ is the fuzzy $O$-preorder relation.
	 Furthermore, according to Proposition~\ref{p:4.3}(1), we obtain that $\underline{R}_{O}^{\beta}(I_O(\alpha_{X},A))=I_O(\alpha_{X},\underline{R}_{O}^{\beta}(A))$ for any $\alpha\in [0,1]$  and $A\in\mathscr{F}(X)$. In a similar way, $\overline{R}_{G}^{\beta}(I^G{(\alpha_{X},A)})=I^G(\alpha_{X},\overline{R}_{G}^{\beta}(A))$ holds.		
		\item
		According to item (1) and Proposition~\ref{p:4.3}(1), it can be directly proved.
		
		\item
		According to item (1) and Proposition~\ref{p:4.3}(1), it can be directly proved.
		
		\item 
		Let $x\in X$, by Proposition~\ref{p:4.8}(1), we get $\underline{R}_{O}^{\beta}(A)(x)=g_A(x)$ and $\underline{R}_{O}^{\beta}(B)(x)=g_B(x)$, then
		$$X_i=\{y:[x_{g_A(x)}]_R^O(y)\le A(y)\}\ and\  	X_j=\{y:[x_{g_B(x)}]_R^O(y)\le B(y)\}.$$
		
	    Hence, $X_i\, ,X_j\in\mathscr{F}_{\beta}(X)$ by Proposition~\ref{p:3.1}, we have $X_i\cap
		X_j\in\mathscr{F}_{(2\beta -1)}(X)$. It holds that for all $y\in X_i\cap X_j$,
		$$[x_{(g_A(x)\land g_B(x))}]_R^O(y)=O(R(x,y),g_A(x))\land O(R(x,y),g_B(x))\le A(y)\land B(y).$$
		So $[x_{(g_A(x)\land g_B(x))}]_R^O\subseteq\underline{R}_{O}^{(2\beta-1)}(A\cap B)$. Since $O$ has 1 as identity element and $R$ is reflexive, we conclude that, 	
		\begin{align*}
		\underline{R}_{O}^{\beta}(A)(x)\land \underline{R}_{O}^{\beta}(B)(x)
		=g_A(x)\land g_B(x)=  [x_{(g_A(x)\land g_B(x))}]_R^O(x)\le\underline{R}_{O}^{(2\beta-1)}(A\cap B) (x),
		\end{align*}
		then $\underline{R}_{O}^{\beta}(A)\cap\underline{R}_{O}^{\beta}(B)\subseteq\underline{R}_{O}^{(2\beta-1)}(A\cap B)$. In a similar way, we get that $\overline{R}_{G}^{(2\beta-1)}(A\cup B)\subseteq\overline{R}_{G}^{\beta}(A)\cup\overline{R}_{G}^{\beta}(B)$. The rest can be proved from Proposition~\ref{p:4.1}(2).

	\end{enumerate}
\end{proof}

Considering special fuzzy relations, we will further explore the characteristics of $\underline{R}_{O}^{\beta}(\underline{R}_{O}^{\beta}(A))\ $and $\ \overline{R}_{G}^{\beta}(\overline{R}_{G}^{\beta}(A))$.
\begin{proposition}\label{p:4.10}
	Let $\alpha=\bigwedge_{x\in X}R(x,x)$, overlap funtion $O$ and grouping function $G$ satisfy \textnormal{(O6)} and \textnormal{(G6)}, respectively. For any $A\in\mathscr{F}(X)$, the following statements hold.
	\begin{enumerate}[(1)]
		\item
		$I\!f\ R(x,y)\le I_O(A(x),A(y))\ for\ all\ x,y\in X,\ then$
		\begin{align*}
		\underline{R}_{O}^{\beta}(O(\alpha_{X},A))\subseteq\underline{R}_{O}^{\beta}(\underline{R}_{O}^{\beta}(A)).
		\end{align*}
		\item
		$I\!f\ O\ and\ G\ are\ dual\ w.r.t.\ N,\ and\ R^N(x,y)\ge I^G(A(x),A(y))\ for\ all\ x,y\in X,\ then$
		\begin{align*}
		\overline{R}_{G}^{\beta}(\overline{R}_{G}^{\beta}(A))\subseteq\overline{R}_{G}^{\beta}(G((\alpha_{X})^{N},A)).
		\end{align*}
	\end{enumerate}
\end{proposition}
\begin{proof}
	\begin{enumerate}[(1)]
		\item
		Let $X_i\in\mathscr{F}_{\beta}(X),\ B=O(\alpha_{X},A)$ and $C=\underline{R}_{O}^{\beta}(A)$. From Lemma~\ref{l:nocondition}(2) and \ref{l:associative law}(2) that for any $x\in X$,
		\begin{align*}
		I_O(g_B^{(i)}(x),g_C^{(i)}(x))
		&=I_O(g_B^{(i)}(x),\underset{y\in X_i}{\bigwedge}I_O(R(x,y),C(y)))\\
		&=\underset{y\in X_i}{\bigwedge} I_O(g_B^{(i)}(x),I_O(R(x,y),C(y)))\\
		&\ge\underset{y\in X_i}{\bigwedge} I_O(I_O(R(x,y),O(\alpha,A(y))), I_O(R(x,y),C(y)))\\
		&\ge\underset{y\in X_i}{\bigwedge} I_O(O(\alpha,A(y)), C(y))) \\
		&\ge\underset{y\in X_i}{\bigwedge} I_O(O(\alpha,A(y)),O(R(y,y),g_A^{(i)}(y))) \\
		&\ge\underset{y\in X_i}{\bigwedge} I_O(O(R(y,y),A(y)),O(R(y,y),g_A^{(i)}(y))) \\
		&\ge\underset{y\in X_i}{\bigwedge}  I_O(A(y),g_A^{(i)}(y))\\
		&=\underset{y\in X_i}{\bigwedge}\ \underset{z\in X_i}{\bigwedge}I_O(A(y),I_O(R(y,z),A(z)))\\
		&=\underset{y\in X_i}{\bigwedge}\ \underset{z\in X_i}{\bigwedge}I_O(R(y,z),I_O(A(y),A(z)))\\
		&= 1.
		\end{align*}
		It follows Lemma~\ref{l:neutal}(2) that $g_B^{(i)}(x)\subseteq g_C^{(i)}(x)$. Thus we have $	\underline{R}_{O}^{\beta}(O(\alpha_{X},A))\subseteq 	\underline{R}_{O}^{\beta}(	\underline{R}_{O}^{\beta}(A))$.
		\item
		According to item (1) and Proposition~\ref{p:3.3}, it can be directly proved.	
	\end{enumerate}
\end{proof}

At the end of this section, sufficient and necessary conditions for $(O,G)$-GVPFRSs to be equal under two different fuzzy relations are given.
\begin{lemma}\label{l:different relations}
		Let $S,R$ be fuzzy $O$-preorder relations, $S\subseteq R$, overlap funtion $O$ and grouping function $G$ satisfy \textnormal{(O6)} and \textnormal{(G6)}, respectively. If $O$ and $G$ are dual w.r.t. $N$, then the following statements hold.
	\begin{align*}
	\underline{R}_{O}^{\beta}(A)\subseteq\underline{S}_{O}^{\beta}(A)\ and\ \overline{S}_{G}^{\beta}(A)\subseteq\overline{R}_{G}^{\beta}(A).
	\end{align*}
\end{lemma}
\begin{proof}
	According to Lemma~\ref{l:nocondition}(4) and Proposition~\ref{p:4.8}, it can be directly proved.	
\end{proof}

\begin{proposition}\label{p:4.11}
	Let $R$ satisfy fuzzy $O$-transitivity, overlap function $O$ and grouping function $G$ satisfy \textnormal{(O6)} and \textnormal{(G6)}, respectively. If $O$ and $G$ are dual w.r.t. $N$, then the following statements hold.
	\begin{enumerate}[(1)]
		\item
		$I\!f\ \underline{S}_{O}^{\beta}(A)(x)=\underline{R}_{O}^{\beta}(A)(x),\ then\ \{y:[x_{\underline{S}_{O}^{\beta}(A)(x)}]^O_R(y)\le A(y)\}\in\mathscr{F}_{\beta}(X).$
		
		\item
		$I\!f\ \overline{S}_{G}^{\beta}(A)(x)=\overline{R}_{G}^{\beta}(A)(x),\ then\ \{y:A(y)\le [x_{\overline{S}_{G}^{\beta}(A)(x)}]^G_R(y) \}\in\mathscr{F}_{\beta}(X).$
	\end{enumerate}
\end{proposition}

\begin{proof}\label{pr:4.11}
	\begin{enumerate}[(1)]
		\item
		Combining Proposition~\ref{p:4.7}(1) and $\underline{S}_{O}^{\beta}(A)(x)=\underline{R}_{O}^{\beta}(A)(x)$, we conclude that

		\begin{align*}
		\{y:[x_{\underline{S}_{O}^{\beta}(A)(x)}]^O_R(y)\le
		A(y)\}=\{y:[x_{\underline{R}_{O}^{\beta}(A)(x)}]^O_R(y)\le A(y)\}\supseteq\{y:[x_{g_A(x)}]^O_R(y)\le A(y)\}.
		\end{align*}
		
		Hence, it follows Proposition~\ref{p:3.1} that $\{y:[x_{\underline{S}_{O}^{\beta}(A)(x)}]^O_R(y)\le A(y)\}\in\mathscr{F}_{\beta}(X)$.
		
		\item
		The proof is similar as item (1).
	\end{enumerate}
\end{proof}

\begin{proposition}{\label{p:4.12}}
	Let $S,R$ be fuzzy $O$-preorder relations, $S\subseteq R$, overlap function $O$ and grouping function $G$ satisfy \textnormal{(O6)} and \textnormal{(G6)}, respectively. If $O$ and $G$ are dual w.r.t. $N$, then the following statements hold.    	
	\begin{enumerate}[(1)]
		\item
		$I\!f\ \{y:[x_{\underline{S}_{O}^{\beta}(A)(x)}]^O_R(y)\le A(y)\}\in\mathscr{F}_{\beta}(X)\ for\ all\ x\in X,\ then\ \underline{S}_{O}^{\beta}(A)=\underline{R}_{O}^{\beta}(A).$
		\item
		$I\!f\ \{y:A(y)\le [x_{\overline{S}_{G}^{\beta}(A)(x)}]^G_R(y) \}\in\mathscr{F}_{\beta}(X)\ for\ all\ x\in X,\ then\ \overline{S}_{G}^{\beta}(A)=\overline{R}_{G}^{\beta}(A).$
	\end{enumerate}  	
\end{proposition}
\begin{proof}
	\begin{enumerate}[(1)]
		\item
		Let $X_i=\{y:[x_{\underline{S}_{O}^{\beta}(A)(x)}]^O_R(y)\le A(y)\}$, then $X_{i}\in \mathscr{F}_{\beta}(X)$ and for all $x\in X,$
		\begin{align*}
		\underline{S}_{O}^{\beta}(A)(x)\le\underset{y\in X_i}{\bigwedge}I_O(R(x,y),A(y))=g^{(i)}_A(x)\le g_A(x)=\underline{R}_{O}^{\beta}(A)(x).
		\end{align*}
		Hence, we obtain $\underline{S}_{O}^{\beta}(A)=\underline{R}_{O}^{\beta}(A)$ by Lemma~\ref{l:different relations}.
		\item
		The proof is similar as item (1).
	\end{enumerate}
\end{proof}
Combining the two propositions above, the following conclusion holds.
\begin{proposition}\label{p:4.13}
	Let $S,R$ be fuzzy $O$-preorder relations, $S\subseteq R$, overlap function $O$ and grouping function $G$ satisfy \textnormal{(O6)} and \textnormal{(G6)}, respectively. If $O$ and $G$ are dual w.r.t. $N$, then the following statements hold.
	\begin{align*}
	\underline{S}_{O}^{\beta}(A)=\underline{R}_{O}^{\beta}(A)\quad\iff\quad\{y:[x_{\underline{S}_{O}^{\beta}(A)(x)}]^O_R(y)\le A(y)\}\in\mathscr{F}_{\beta}(X),\\
	\overline{S}_{G}^{\beta}(A)=\overline{R}_{G}^{\beta}(A)\quad\iff\quad\{y:A(y)\le [x_{\overline{S}_{G}^{\beta}(A)(x)}]^G_R(y) \}\in\mathscr{F}_{\beta}(X).
	\end{align*}
	Furthermore, when fuzzy sets are taken as crisp sets, the following conclusions hold.
\end{proposition}
\begin{proposition}\label{p:4.14}
	Let $S,R$ be fuzzy $O$-preorder relation, $S\subseteq R$, overlap function $O$ and grouping function $G$ satisfy \textnormal{(O6)} and \textnormal{(G6)}. If $O$ and $G$ are dual w.r.t. $N$, for all crisp set $A$, the following statements hold .
	\begin{align*}
	\underline{S}_{O}^{\beta}(A)=\underline{R}_{O}^{\beta}(A)\quad &\iff\quad|\{y:y\notin A,[x_{\underline{S}_{O}^{\beta}(A)(x)}]^O_R(y)=0\}|\ge\beta|X|-|A|,\\
	\overline{S}_{G}^{\beta}(A)=\overline{R}_{G}^{\beta}(A)\quad &\iff\quad |\{y:y\in A, [x_{\overline{S}_{G}^{\beta}(A)(x)}]^G_R(y)=1 \}|\ge|A|+(\beta-1)|X|.
	\end{align*}
\end{proposition}

\begin{proof}\label{pr:4.14}
	For any crisp set $A$, we can conclude that,
	\begin{align*}
	\{y:[x_{\underline{S}_{O}^{\beta}(A)(x)}]^O_R(y)\le A(y)\}=A\bigcup\{y:y\notin A,[x_{\underline{S}_{O}^{\beta}(A)(x)}]^O_R(y)=0\}.
	\end{align*}
	Hence, according to Proposition~\ref{p:4.13} that
	\begin{align*}
	\underline{S}_{O}^{\beta}(A)=\underline{R}_{O}^{\beta}(A)\quad &\iff\quad|\{y:y\notin A,[x_{\underline{S}_{O}^{\beta}(A)(x)}]^O_R(y)=0\}\ge\beta|X|-|A|.
	\end{align*}
	The equivalent expression about $G$ can be proven in a similar way. 	
\end{proof}

\section{Conclusions}\label{section6}
In this paper, a new type of fuzzy rough set model on arbitrary fuzzy relations was defined by using overlap and grouping functions, which called $(O,G)$-GVPFRSs.
Meanwhile, we gave two equivalent expressions of the upper and lower approximation operators applying fuzzy implications and co-implications, which facilitate more efficient calculations.
 In particular, some special conclusions were further discussed, when fuzzy relations and sets degenerated to crisp relations and sets.
 In addition, we characterized the $(O,G)$-GVPFRSs based on diverse fuzzy relations.
Finally, the richer conclusions about $(O,G)$-GVPFRSs were gave under some addtional conditions.
In general, this paper further explored the GVPFRSs from a theoretical perspective based on overlap and grouping functions.

\section*{Acknowledgements}
This research was supported by the National Natural Science Foundation of China (Grant nos. 11901465, 12101500), the Science and Technology Program of Gansu Province (20JR10RA101), the Scientific Research Fund for Young Teachers of Northwest Normal University (NWNU-LKQN-18-28), the Doctoral Research Fund of Northwest Normal University (6014/0002020202) and the Chinese Universities Scientific Fund (Grant no. 2452018054).
\section*{Conflict of interests}
The authors declare that there are no conflict of interests.



\end{document}